\documentclass{article}

\PassOptionsToPackage{numbers, compress}{natbib}
\usepackage[final]{nips_2017}

\usepackage[utf8]{inputenc} 
\usepackage[T1]{fontenc}    
\usepackage{url}            
\usepackage{booktabs}       
\usepackage{amsfonts}       
\usepackage{nicefrac}       
\usepackage{microtype}      
\usepackage{amsmath}
\usepackage[inline]{enumitem}
\usepackage{mathtools}
\usepackage{color}
\usepackage{units}
\usepackage{amsthm}
\usepackage{thmtools}
\usepackage{thm-restate}

\usepackage{algorithm}
\usepackage[noend]{algorithmic}
\usepackage{subfigure}

\usepackage[hidelinks]{hyperref}
\usepackage[capitalise]{cleveref}
\usepackage[export]{adjustbox}

\renewcommand{\paragraph}[1]{\textbf{#1}~~}

\theoremstyle{plain}
\newtheorem{theorem}{Theorem}
\newtheorem{lemma}{Lemma}

\newtheorem{corollary}{Corollary}

\newtheorem{assumption}{Assumption}

\newtheorem{remark}{Remark}

\crefname{equation}{}{} 
\crefname{section}{Sec.}{Sec.}
\crefname{enumi}{}{}
\crefname{assumption}{Assumption}{Assumptions}

\newcommand{\mb}[1]{\ensuremath{\mathbf{#1}}}
\newcommand{\xb}{\ensuremath{\mb{x}}}
\newcommand{\ab}{\ensuremath{\mb{z}}}
\newcommand{\ub}{\ensuremath{\mb{u}}}

\newcommand{\argmax}{\operatornamewithlimits{argmax}}
\newcommand{\argmin}{\operatornamewithlimits{argmin}}

\newcommand{\trace}{\ensuremath{\mathrm{trace}}}
\newcommand{\T}{\mathrm{T}}

\newcommand{\ndata}{\ensuremath{n}}
\newcommand{\ti}{\ensuremath{t}}
\newcommand{\nstate}{\ensuremath{q}}
\newcommand{\ninp}{\ensuremath{p}}

\newcommand{\Rlevel}{\ensuremath{R^\mathrm{lev}}}
\newcommand{\Rsafe}{\ensuremath{R^\mathrm{dec}}}

\newcommand{\Reps}{\ensuremath{R_\epsilon}}
\newcommand{\Rbar}{\ensuremath{\overline{R}\vphantom{R}}}

\newcommand{\DecSet}{\ensuremath{\mathcal{D}}}
\newcommand{\SafeSet}{\ensuremath{\mathcal{S}}}

\newcommand{\eqdef}{\mathrel{:}=}

\graphicspath{{figures/}}

\title{Safe Model-based Reinforcement Learning with Stability Guarantees}

\author{
  Felix Berkenkamp \\
  Department of Computer Science \\
  ETH Zurich \\
  \texttt{befelix@inf.ethz.ch} \\
  \And
  Matteo Turchetta \\
  Department of Computer Science,  \\
  ETH Zurich \\
  \texttt{matteotu@inf.ethz.ch} \\
  \AND
  Angela P. Schoellig \\
  Institute for Aerospace Studies \\
  University of Toronto \\
  \texttt{schoellig@utias.utoronto.ca} \\
  \And
  Andreas Krause \\
  Department of Computer Science \\
  ETH Zurich \\
  \texttt{krausea@ethz.ch}
}

\begin{document}

\maketitle


\begin{abstract}
Reinforcement learning is a powerful paradigm for learning optimal policies from experimental data. However, to find optimal policies, most reinforcement learning algorithms explore all possible actions, which may be harmful for real-world systems. As a consequence, learning algorithms are rarely applied on safety-critical systems in the real world. In this paper, we present a learning algorithm that explicitly considers safety, defined in terms of stability guarantees. Specifically, we extend control-theoretic results on Lyapunov stability verification and show how to use statistical models of the dynamics to obtain high-performance control policies with provable stability certificates. Moreover, under additional regularity assumptions in terms of a Gaussian process prior, we prove that one can effectively and safely collect data in order to learn about the dynamics and thus both improve control performance and expand the safe region of the state space. In our experiments, we show how the resulting algorithm can safely optimize a neural network policy on a simulated inverted pendulum, without the pendulum ever falling down.
\end{abstract}


\section{Introduction}
\label{sec:introduction}

While reinforcement learning (RL,~\cite{Sutton1998Reinforcement}) algorithms have achieved impressive results in games, for example on the Atari platform~\citep{Mnih2015Humanlevel}, they are rarely applied to real-world physical systems (e.g., robots) outside of academia. The main reason is that RL algorithms provide optimal policies only in the long-term, so that intermediate policies may be unsafe, break the system, or harm their environment. This is especially true in safety-critical systems that can affect human lives. Despite this, safety in RL has remained largely an open problem~\citep{Amodei2016Concrete}.

Consider, for example, a self-driving car. While it is desirable for the algorithm that drives the car to improve over time (e.g., by adapting to driver preferences and changing environments), any policy applied to the system has to guarantee safe driving. Thus, it is not possible to learn about the system through random exploratory actions, which almost certainly lead to a crash. In order to avoid this problem, the learning algorithm needs to consider its ability to safely recover from exploratory actions. In particular, we want the car to be able to recover to a safe state, for example, driving at a reasonable speed in the middle of the lane. This ability to recover is known as \emph{asymptotic stability} in control theory~\citep{Khalil1996Nonlinear}. Specifically, we care about the \emph{region of attraction} of the closed-loop system under a policy. This is a subset of the state space that is forward invariant so that any state trajectory that starts within this set stays within it for all times and converges to a goal state eventually.

In this paper, we present a RL algorithm for continuous state-action spaces that provides these kind of high-probability safety guarantees for policies. In particular, we show how, starting from an initial, safe policy we can expand our estimate of the region of attraction by collecting data inside the safe region and adapt the policy to both increase the region of attraction and improve control performance.

\paragraph{Related work}
Safety is an active research topic in RL and different definitions of safety exist~\citep{Pecka2014Safe,Garcia2015Comprehensive}. \emph{Discrete} Markov decision processes (MDPs) are one class of tractable models that have been analyzed. In risk-sensitive RL, one
specifies risk-aversion in the reward~\citep{Coraluppi1999Risksensitive}. For example, \citep{Geibel2005RiskSensitive} define risk as the probability of driving the agent to a set of known, undesirable states. Similarly, robust MDPs maximize rewards when transition probabilities are uncertain~\citep{Tamar2014Scaling,Wiesemann2012Robust}. Both \citep{Moldovan2012Safe} and \citep{Turchetta2016Safe} introduce algorithms to safely explore MDPs so that the agent never gets stuck without safe actions. All these methods require an accurate probabilistic model of the system.

In \emph{continuous} state-action spaces, model-free policy search algorithms have been successful. These update policies without a system model by repeatedly executing the same task~\citep{Peters2006Policy}. In this setting,~\citep{Achiam2017Constrained} introduces safety guarantees in terms of constraint satisfaction that hold in expectation. High-probability worst-case safety guarantees are available for methods based on Bayesian optimization~\citep{Mockus1989Bayesian} together with Gaussian process models (GP,~\citep{Rasmussen2006Gaussian}) of the cost function. The algorithms in~\citep{Schreiter2015Safe} and~\citep{Sui2015Safe} provide high-probability safety guarantees for any parameter that is evaluated on the real system. These methods are used in~\citep{Berkenkamp2016Safe} to safely optimize a parametric control policy on a quadrotor. However, resulting policies are task-specific and require the system to be reset.

In the \emph{model-based} RL setting, research has focused on safety in terms of state constraints. In \citep{Garcia2012Safe,Hans2008Safe}, \textit{a priori} known, safe global backup policies are used, while~\citep{Perkins2003Lyapunov} learns to switch between several safe policies. However, it is not clear how one may find these policies in the first place. Other approaches use model predictive control with constraints, a model-based technique where the control actions are optimized online. For example,~\citep{Sadigh2016Safe} models uncertain environmental constraints, while~\citep{Ostafew2016Robust} uses approximate uncertainty propagation of GP dynamics along trajectories. In this setting, robust feasability and constraint satisfaction can be guaranteed for a learned model with bounded errors using robust model predictive control~\cite{Aswani2013Provably}. The method in~\citep{Akametalu2014Reachability} uses reachability analysis to construct safe regions in the state space. The theoretical guarantees depend on the solution to a partial differential equation, which is approximated.

Theoretical guarantees for the stability exist for the more tractable stability analysis and verification under a \emph{fixed} control policy. In control, stability of a known system can be verified using a Lyapunov function~\citep{Bobiti2016Sampling}. A similar approach is used by~\citep{Berkenkamp2016Lyapunov} for deterministic, but unknown dynamics that are modeled as a GP, which allows for provably safe learning of regions of attraction for fixed policies. Similar results are shown in \citep{Vinogradska2016Stability} for stochastic systems that are modeled as a GP\@. They use Bayesian quadrature to compute provably accurate estimates of the region of attraction. These approaches do not update the policy.

\paragraph{Our contributions}
We introduce a novel algorithm that can safely optimize policies in continuous state-action spaces while providing high-probability safety guarantees in terms of stability. Moreover, we show that it is possible to exploit the regularity properties of the system in order to \emph{safely learn} about the dynamics and thus improve the policy and increase the estimated safe region of attraction without ever leaving it. Specifically, starting from a policy that is known to stabilize the system locally, we gather data at informative, safe points and improve the policy safely based on the improved model of the system and prove that any exploration algorithm that gathers data at these points reaches a natural notion of \emph{full exploration}. We show how the theoretical results transfer to a \emph{practical algorithm} with safety guarantees and apply it to a simulated inverted pendulum stabilization task.


\section{Background and Assumptions}
\label{sec:problem}

We consider a deterministic, discrete-time dynamic system
\begin{equation}
  \xb_{\ti+1} = f(\xb_\ti, \ub_\ti) = h(\xb_\ti, \ub_\ti) + g(\xb_\ti, \ub_\ti),
\label{eq:dynamic_system}
\end{equation}
with states ${\xb \in \mathcal{X} \subset \mathbb{R}^\nstate}$ and control actions ${\ub \in \mathcal{U} \subset \mathbb{R}^\ninp}$ and a discrete time index~${t \in \mathbb{N}}$. The true dynamics~$f \colon \mathcal{X} \times \mathcal{U} \to \mathcal{X}$ consist of two parts: $h(\xb_\ti, \ub_\ti)$ is a known, prior model that can be obtained from first principles, while $g(\xb_\ti, \ub_\ti)$ represents \textit{a priori} unknown model errors. While the model errors are unknown, we can obtain noisy measurements of $f(\xb, \ub)$ by driving the system to the state~$\xb$ and taking action~$\ub$. We want this system to behave in a certain way, e.g., the car driving on the road. To this end, we need to specify a control policy~$\pi \colon \mathcal{X} \to \mathcal{U}$ that, given the current state, determines the appropriate control action that drives the system to some goal state, which we set as the origin without loss of generality~\citep{Khalil1996Nonlinear}. We encode the performance requirements of how to drive the system to the origin through a positive cost~$r(\xb,\ub)$ that is associated with states and actions and has~$r(\mathbf{0}, \mathbf{0}) = 0$. The policy aims to  minimize the cumulative, discounted costs for each starting state.

The goal is to safely learn about the dynamics from measurements and adapt the policy for performance, without encountering system failures. Specifically, we define the safety constraint on the state divergence that occurs when leaving the region of attraction. This means that adapting the policy is not allowed to decrease the region of attraction and exploratory actions to learn about the dynamics~$f(\cdot)$ are not allowed to drive the system outside the region of attraction. The region of attraction is \textit{not known a priori}, but is implicitly defined through the system dynamics and the choice of policy. Thus, the policy not only defines performance as in typical RL, but also determines safety and where we can obtain measurements.

\paragraph{Model assumptions}
In general, this kind of safe learning is impossible without further assumptions. For example, in a discontinuous system even a slight change in the control policy can lead to drastically different behavior. Moreover, to expand the safe set we need to generalize learned knowledge about the dynamics to (potentially unsafe) states that we have not visited. To this end, we restrict ourselves to the general and practically relevant class of models that are Lipschitz continuous. This is a typical assumption in the control community~\cite{Khalil1996Nonlinear}.
Additionally, to ensure that the closed-loop system remains Lipschitz continuous when the control policy is applied, we restrict policies to the rich class of $L_\pi$-Lipschitz continuous functions~$\Pi_L$, which also contains certain types of neural networks~\citep{Szegedy2014Intriguing}.
\begin{assumption}[continuity]
\label{as:lipschitz_continuity}
  The dynamics $h(\cdot)$ and $g(\cdot)$ in~\cref{eq:dynamic_system} are $L_h$- and $L_g$ Lipschitz continuous with respect to the 1-norm. The considered control policies~$\pi$ lie in a set~$\Pi_L$ of functions that are~$L_\pi$-Lipschitz continuous with respect to the 1-norm.
\end{assumption}
To enable safe learning, we require a reliable statistical model. While we commit to GPs for the exploration analysis, for safety any suitable, well-calibrated model is applicable.
\begin{assumption}[well-calibrated model]
\label{as:f_confidence_interval}
Let~$\mu_\ndata(\cdot)$ and~$\Sigma_\ndata(\cdot)$ denote the posterior mean and covariance matrix functions of the statistical model of the dynamics~\cref{eq:dynamic_system} conditioned on~$\ndata$ noisy measurements. With~${\sigma_\ndata(\cdot) = \trace(\Sigma^{1/2}_\ndata(\cdot))}$, there exists a~$\beta_\ndata>0$ such that with probability at least $(1-\delta)$ it holds for all~${\ndata \geq 0}$, ${\xb \in \mathcal{X}}$, and~${\ub \in \mathcal{U}}$ that
${
  \| f(\xb, \ub) - \mu_{\ndata}(\xb, \ub) \|_1 \leq \beta_\ndata \sigma_{\ndata}(\xb, \ub).
}$
\end{assumption}
This assumption ensures that we can build confidence intervals on the dynamics that, when scaled by an appropriate constant~$\beta_\ndata$, cover the true function with high probability. We introduce a specific statistical model that fulfills both assumptions under certain regularity assumptions in~\cref{sec:stability}.

\paragraph{Lyapunov function}
To satisfy the specified safety constraints for safe learning, we require a tool to determine whether individual states and actions are safe. In control theory, this safety is defined through the region of attraction, which can be computed for a fixed policy using Lyapunov functions~\citep{Khalil1996Nonlinear}. Lyapunov functions are continuously differentiable functions~${v \colon \mathcal{X} \to \mathbb{R}_{\geq 0} }$ with~${v(\mb{0}) = 0}$ and ${ v(\xb) > 0 }$ for all ${ \xb \in \mathcal{X} \setminus \{\mb{0}\} }$. The key idea behind using Lyapunov functions to show stability of the system~\cref{eq:dynamic_system} is similar to that of gradient descent on strictly quasiconvex functions: if one can show that, given a policy~$\pi$, applying the dynamics $f$ on the state maps it to strictly smaller values on the Lyapunov function (`going downhill'), then the state eventually converges to the  equilibrium point at the origin (minimum). In particular, the assumptions in~\cref{thm:lyapunov_stability} below imply that $v$ is strictly quasiconvex within the region of attraction if the dynamics are Lipschitz continuous. As a result, the one step decrease property for all states within a level set guarantees eventual convergence to the origin.
\begin{theorem}[\citep{Khalil1996Nonlinear}]
\label{thm:lyapunov_stability}
\label{thm:region_of_attraction}
  Let $v$ be a Lyapunov function, $f$ Lipschitz continuous dynamics, and $\pi$ a policy. If ${ v(f(\xb, \pi(\xb))) < v(\xb) }$ for all $\xb$ within the level set $\mathcal{V}(c) = \{ \xb \in \mathcal{X} \setminus \{\mb{0} \} \,|\, v(\xb) \leq c \}$, ${c>0}$, then $\mathcal{V}(c)$ is a region of attraction, so that $\xb_0 \in \mathcal{V}(c)$ implies $\xb_\ti \in \mathcal{V}(c)$ for all $\ti > 0$ and $\lim_{t \to \infty} \xb_\ti = \mb{0}$.
\end{theorem}
It is convenient to characterize the region of attraction through a level set of the Lyapunov function, since it replaces the challenging test for convergence with a one-step decrease condition on the Lyapunov function. For the theoretical analysis in this paper, we assume that a Lyapunov function is given to determine the region of attraction. For ease of notation, we also assume ${\partial v(\xb) / \partial \xb \neq \mb{0} }$ for all ${\xb \in \mathcal{X} \setminus \mb{0}}$, which ensures that level sets~$\mathcal{V}(c)$ are connected if~${c>0}$. Since Lyapunov functions are continuously differentiable, they are $L_v$-Lipschitz continuous over the compact set $\mathcal{X}$.

In general, it is not easy to find suitable Lyapunov functions. However, for physical models, like the prior model $h$ in~\cref{eq:dynamic_system}, the energy of the system (e.g., kinetic and potential for mechanical systems) is a good candidate Lyapunov function. Moreover, it has recently been shown that it is possible to compute suitable Lyapunov functions~\citep{Li2016Computation,Giesl2015Review}. In our experiments, we exploit the fact that value functions in RL are Lyapunov functions if the costs are strictly positive away from the origin. This follows directly from the definition of the value function, where $v(\xb) = r(\xb,\pi(\xb)) + v(f(\xb, \pi(\xb)) \leq v(f(\xb, \pi(\xb)))$. Thus, we can obtain Lyapunov candidates as a by-product of approximate dynamic programming.

\paragraph{Initial safe policy}
Lastly, we need to ensure that there exists a safe starting point for the learning process. Thus, we assume that we have an initial policy~$\pi_0$ that renders the origin of the system in~\cref{eq:dynamic_system} asymptotically stable within some small set of states~$\mathcal{S}^x_0$. For example, this policy may be designed using the prior model~$h$ in~\cref{eq:dynamic_system}, since most models are locally accurate but deteriorate in quality as state magnitude increases. This policy is explicitly \textit{not safe} to use throughout the state space~$\mathcal{X} \setminus \mathcal{S}^x_0$.


\section{Theory}
\label{sec:stability}

In this section, we use these assumptions for safe reinforcement learning. We start by computing the region of attraction for a fixed policy under the statistical model. Next, we optimize the policy in order to expand the region of attraction. Lastly, we show that it is possible to safely learn about the dynamics and, under additional assumptions about the model and the system's reachability properties, that this approach expands the estimated region of attraction safely. We consider an idealized algorithm that is amenable to analysis, which we convert to a practical variant in~\cref{sec:experiments}. See~\cref{fig:example_set} for an illustrative run of the algorithm and examples of the sets defined below.

\paragraph{Region of attraction}
We start by computing the region of attraction for a fixed policy.
This is an extension of the method in~\citep{Berkenkamp2016Lyapunov} to discrete-time systems. We want to use the Lyapunov decrease condition in~\cref{thm:region_of_attraction} to guarantee safety for the statistical model of the dynamics. However, the posterior uncertainty in the statistical model of the dynamics means that one step predictions about~$v(f(\cdot))$ are uncertain too. We account for this by constructing high-probability confidence intervals on~$v(f(\xb, \ub))$:
$
  \mathcal{Q}_\ndata(\xb, \ub)
  \eqdef
   [ v(\mu_{\ndata-1}(\xb, \ub)) \pm L_v \beta_\ndata \sigma_{\ndata-1}(\xb, \ub) ] .
$
From~\cref{as:f_confidence_interval} together with the Lipschitz property of~$v$, we know that $v(f(\xb, \ub))$ is contained in $\mathcal{Q}_\ndata(\xb, \ub)$ with probability at least~$(1 - \delta)$. For our exploration analysis, we need to ensure that safe state-actions cannot become unsafe; that is, an initial set of safe set~$\mathcal{S}_0$ remains safe (defined later).
To this end, we intersect the confidence intervals: ${\mathcal{C}_\ndata(\xb, \ub) \eqdef \mathcal{C}_{\ndata - 1} \cap \mathcal{Q}_\ndata(\xb, \ub)}$, where the set~$\mathcal{C}$ is initialized to ${\mathcal{C}_0(\xb, \ub) = (-\infty, v(\xb) - L_{\Delta v} \tau ) }$ when ${(\xb, \ub) \in \SafeSet_0}$ and ${\mathcal{C}_0(\xb, \ub) = \mathbb{R}}$ otherwise.
Note that $v(f(\xb, \ub))$ is contained in~$\mathcal{C}_\ndata(\xb, \ub)$ with the same $(1-\delta)$ probability as in~\cref{as:f_confidence_interval}.
The upper and lower bounds on $v(f(\cdot))$ are defined as $u_\ndata(\xb, \ub) \eqdef \max \mathcal{C}_\ndata(\xb, \ub)$ and $l_\ndata(\xb, \ub) \eqdef \min \mathcal{C}_\ndata(\xb, \ub)$.

Given these high-probability confidence intervals, the system is stable according to~\cref{thm:lyapunov_stability} if $v(f(\xb, \ub)) \leq u_n(\xb) < v(\xb) $ for all~$\xb \in \mathcal{V}(c)$. However, it is intractable to verify this condition directly on the continuous domain without additional, restrictive assumptions about the model. Instead, we consider a discretization of the state space~$\mathcal{X}_\tau \subset \mathcal{X}$ into cells, so that $\| \xb - {[\xb]}_\tau \|_1 \leq \tau$ holds for all $\xb \in \mathcal{X}$. Here, ${[\xb]}_\tau$ denotes the point in $\mathcal{X}_\tau$ with the smallest $l_1$ distance to~$\xb$. Given this discretization, we bound the decrease variation on the Lyapunov function for states in~$\mathcal{X}_\tau$ and use the Lipschitz continuity to generalize to the continuous state space~$\mathcal{X}$.

\begin{restatable}{theorem}{gproa}
  Under~\cref{as:lipschitz_continuity,as:f_confidence_interval} with~${L_{\Delta v} \eqdef L_v L_f (L_\pi + 1) + L_v}$, let~$\mathcal{X}_\tau$ be a discretization of $\mathcal{X}$ such that ${\| \xb - {[\xb]}_\tau \|_1 \leq \tau }$ for all~ ${\xb \in \mathcal{X}}$.
    If, for all ${\xb \in \mathcal{V}(c) \cap \mathcal{X}_\tau}$ with $c>0$, $\ub = \pi(\xb)$, and for some~$n \geq 0$ it holds that
${
  u_\ndata(\xb, \ub)  < v(\xb) - L_{\Delta v} \tau,
}$
then $v(f(\xb, \pi(\xb))) < v(\xb)$ holds for all ${\xb \in \mathcal{V}(c)}$ with probability at least $(1 - \delta)$ and $\mathcal{V}(c)$ is a region of attraction for~\cref{eq:dynamic_system} under policy~$\pi$.
\label{eq:safety_constraint}
\label{thm:gp_region_of_attraction}
\end{restatable}
The proof is given in~\cref{sec:proofs_stability}. \cref{thm:gp_region_of_attraction} states that, given confidence intervals on the statistical model of the dynamics, it is sufficient to check the stricter decrease condition in~\cref{eq:safety_constraint} on the discretized domain~$\mathcal{X}_\tau$ to guarantee the requirements for the region of attraction in the continuous domain in~\cref{thm:region_of_attraction}. The bound in~\cref{eq:safety_constraint} becomes tight as the discretization constant~$\tau$ and~$|v(f(\cdot)) - u_n(\cdot)|$ go to zero. Thus, the discretization constant trades off computation costs for accuracy, while~$u_n$ approaches~$v(f(\cdot))$ as we obtain more measurement data and the posterior model uncertainty about the dynamics,~$\sqrt{\beta_\ndata} \sigma_\ndata$ decreases. The confidence intervals on~$v(f(\xb, \pi(\xb)) - v(\xb)$ and the corresponding estimated region of attraction (red line) can be seen in the bottom half of~\cref{fig:example_set}.

\begin{figure*}[t]
  \subfigure[Initial safe set (in red).\label{fig:example_set_1}]{\includegraphics[scale=1]{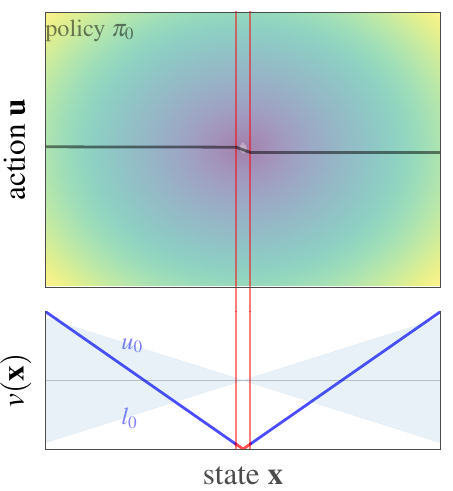}}
  \hfill
  \subfigure[Exploration: 15 data points.\label{fig:example_set_2}]{\includegraphics[scale=1]{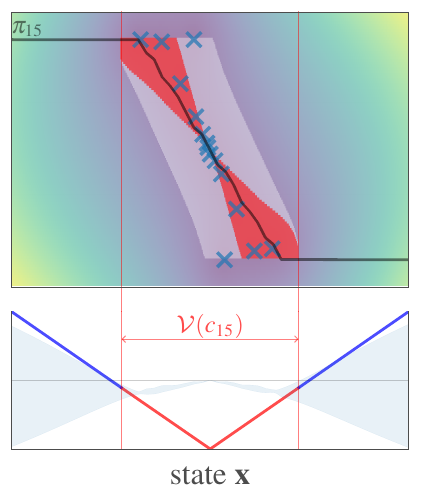}}
  \hfill
  \subfigure[Final policy after 30 evaluations.\label{fig:example_set_3}]{\includegraphics[scale=1]{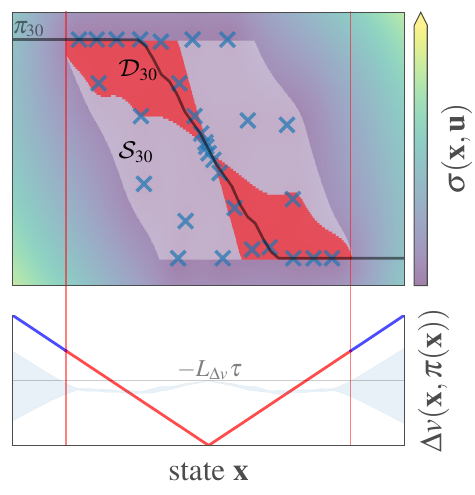}}
  \caption{Example application of~\cref{alg:safe_learning}. Due to input constraints, the system becomes unstable for large states.
  We start from an initial, local policy~$\pi_0$ that has a small, safe region of attraction (red lines) in~\cref{fig:example_set_1}. The algorithm selects safe, informative state-action pairs within~$\mathcal{S}_\ndata$ (top, white shaded), which can be evaluated without leaving the region of attraction~$\mathcal{V}(c_\ndata)$ (red lines) of the current policy~$\pi_\ndata$. As we gather more data (blue crosses), the uncertainty in the model decreases (top, background) and we use~\cref{eq:policy_update} to update the policy so that it lies within~$\DecSet_\ndata$ (top, red shaded) and fulfills the Lyapunov decrease condition. The algorithm converges to the largest safe set in~\cref{fig:example_set_3}. It improves the policy without evaluating unsafe state-action pairs and thereby without system failure.
  }
\label{fig:example_set}
\end{figure*}

\paragraph{Policy optimization}
So far, we have focused on estimating the region of attraction for a fixed policy. Safety is a property of states under a fixed policy. This means that the policy directly determines which states are safe. Specifically, to form a region of attraction all states in the discretizaton~$\mathcal{X}_\tau$ within a level set of the Lyapunov function need to fulfill the decrease condition in~\cref{thm:gp_region_of_attraction} that depends on the policy choice. The set of all state-action pairs that fulfill this decrease condition is given by
\begin{align}
  \DecSet_\ndata = \big\{ (\xb, \ub) \in \mathcal{X}_\tau \times \mathcal{U} \,|\,
  u_\ndata(\xb, \ub) - v(\xb)  < -L_{\Delta v} \tau \big\},
  \label{eq:gp_dec_set_nice}
\end{align}
see~\cref{fig:example_set_3} (top, red shaded). In order to estimate the region of attraction based on this set, we need to commit to a policy. Specifically, we want to pick the policy that leads to the largest possible region of attraction according to~\cref{thm:gp_region_of_attraction}. This requires that for each discrete state in~$\mathcal{X}_\tau$ the corresponding state-action pair under the policy must be in the set~$\DecSet_\ndata$. Thus, we optimize the policy according to
\begin{equation}
  \pi_n, c_n = \argmax_{\pi \in \Pi_L, c \in \mathbb{R}_{>0}} c, \quad \textnormal{such~that~for all~} \xb \in \mathcal{V}(c) \cap \mathcal{X}_\tau \colon  (\xb, \pi(\xb)) \in \mathcal{D}_\ndata .
  \label{eq:policy_update}
\end{equation}
The region of attraction that corresponds to the optimized policy~$\pi_\ndata$ according to~\cref{eq:policy_update} is given by~$\mathcal{V}(c_\ndata)$, see~\cref{fig:example_set_2}. It is the largest level set of the Lyapunov function for which all state-action pairs~$(\xb, \pi_n(\xb))$ that correspond to discrete states within~$\mathcal{V}(c_n) \cap \mathcal{X}_\tau$ are contained in~$\DecSet_\ndata$. This means that these state-action pairs fulfill the requirements of~\cref{thm:gp_region_of_attraction} and~$\mathcal{V}(c_n)$ is a region of attraction of the true system under policy~$\pi_n$. The following theorem is thus a direct consequence of~\cref{thm:gp_region_of_attraction} and~\cref{eq:policy_update}.

\begin{restatable}{theorem}{roanice}
  Let $\mathcal{R}_{\pi_\ndata}$ be the true region of attraction of~\cref{eq:dynamic_system} under the policy~$\pi_\ndata$. For any $\delta \in (0, 1)$, we have with probability at least $(1 - \delta)$ that
$
  {\mathcal{V}(c_\ndata) \subseteq \mathcal{R}_{\pi_\ndata}}
$
for all $\ndata > 0$.
\end{restatable}
Thus, when we optimize the policy subject to the constraint in~\cref{eq:policy_update} the estimated region of attraction is always an inner approximation of the true region of attraction. However, solving the optimization problem in~\cref{eq:policy_update} is intractable in general. We approximate the policy update step in~\cref{sec:experiments}.


\paragraph{Collecting measurements}
Given these stability guarantees, it is natural to ask how one might obtain data points in order to improve the model of~$g(\cdot)$ and thus efficiently increase the region of attraction. This question is difficult to answer in general, since it depends on the property of the statistical model. In particular, for general statistical models it is often not clear whether the confidence intervals contract sufficiently quickly. In the following, we make additional assumptions about the model and reachability within~$\mathcal{V}(c_n)$ in order to provide exploration guarantees. These assumptions allow us to highlight fundamental requirements for safe data acquisition and that safe exploration is possible.

We assume that the unknown model errors~$g(\cdot)$ have bounded norm in a reproducing kernel Hilbert space (RKHS,~\cite{Scholkopf2002Learning}) corresponding to a differentiable kernel~$k$, $\|g(\cdot)\|_k \leq B_g$. These are a class of well-behaved functions of the form~$g(\ab) = \sum_{i=0}^\infty \alpha_i k(\ab_i, \ab)$ defined through representer points~$\ab_i$ and weights~$\alpha_i$ that decay sufficiently fast with~$i$. This assumption ensures that~$g$ satisfies the Lipschitz property in~\cref{as:lipschitz_continuity}, see~\cite{Berkenkamp2016Lyapunov}. Moreover, with~$
\beta_\ndata = B_g + 4 \sigma \sqrt{\gamma_\ndata + 1 + \ln(1 / \delta)}
$
we can use GP models for the dynamics that fulfill~\cref{as:f_confidence_interval} if the state if fully observable and the measurement noise is~$\sigma$-sub-Gaussian (e.g., bounded in~$[-\sigma, \sigma]$), see~\cite{Chowdhury2017Kernelized}. Here~$\gamma_\ndata$ is the information capacity. It corresponds to the amount of mutual information that can be obtained about~$g$ from~$\ndata q$ measurements, a measure of the size of the function class encoded by the model. The information capacity has a sublinear dependence on~$\ndata$ for common kernels and upper bounds can be computed efficiently~\cite{Srinivas2012Gaussian}. More details about this model are given in~\cref{sec:gaussian_process_theory}.

In order to quantify the exploration properties of our algorithm, we consider a discrete action space~${\mathcal{U}_\tau \subset \mathcal{U}}$. We define exploration as the number of state-action pairs in~$\mathcal{X}_\tau \times \mathcal{U}_\tau$ that we can safely learn about without leaving the true region of attraction. Note that despite this discretization, the policy takes values on the continuous domain. Moreover, instead of using the confidence intervals directly as in~\cref{eq:policy_update}, we consider an algorithm that uses the Lipschitz constants to slowly expand the safe set. We use this in our analysis to quantify the ability to generalize beyond the current safe set. In practice, nearby states are sufficiently correlated under the model to enable generalization using~\cref{eq:gp_dec_set_nice}.

Suppose we are given a set~$\SafeSet_0$ of state-action pairs about which we can learn safely. Specifically, this means that we have a policy such that, for any state-action pair $(\xb, \ub)$ in~$\SafeSet_0$, if we apply action $\ub$ in state $\xb$ and then apply actions according to the policy, the state  converges to the origin. Such a set can be constructed using the initial policy $\pi_0$ from~\cref{sec:problem} as~$\SafeSet_0 = \{(\xb,\pi_0(\xb)) \,|\, \xb \in \SafeSet_0^x \}$. Starting from this set, we want to update the policy to expand the region of attraction according to~\cref{thm:gp_region_of_attraction}. To this end, we use the confidence intervals on~$v(f(\cdot))$ for states inside~$\SafeSet_0$ to determine state-action pairs that fulfill the decrease condition. We thus redefine~$\DecSet_\ndata$ for the exploration analysis to
\begin{align}
  \DecSet_\ndata = \bigcup_{ (\xb, \ub) \in \SafeSet_{\ndata - 1}}  \big\{ \ab' \in \mathcal{X}_\tau \times \mathcal{U}_\tau \,|\,
  u_\ndata(\xb, \ub) - v(\xb) + L_{\Delta v} \| \ab' - (\xb, \ub) \|_1 < -L_{\Delta v} \tau \big\}.
  \label{eq:gp_dec_set}
\end{align}
This formulation is equivalent to~\cref{eq:gp_dec_set_nice}, except that it uses the Lipschitz constant to generalize safety. Given~$\DecSet_\ndata$, we can again find a region of attraction~$\mathcal{V}(c_n)$ by committing to a policy according to~\cref{eq:policy_update}. In order to expand this region of attraction effectively we need to decrease the posterior model uncertainty about the dynamics of the GP by collecting measurements. However, to ensure safety as outlined in~\cref{sec:problem}, we are not only restricted to states within~$\mathcal{V}(c_n)$, but also need to ensure that the state after taking an action is safe; that is, the dynamics map the state back into the region of attraction~$\mathcal{V}(c_n)$. We again use the Lipschitz constant in order to determine this set,
\begin{align}
  \SafeSet_\ndata = \bigcup_{\ab \in \SafeSet_{\ndata - 1}} \big\{ \ab' \in \mathcal{V} ( c_n ) \cap \mathcal{X}_\tau \times \mathcal{U}_\tau \,|\,
  u_\ndata(\ab) + L_v L_f \| \ab - \ab' \|_1 \leq c_n \}.
  \label{eq:gp_safe_set}
\end{align}
The set~$\SafeSet_\ndata$ contains state-action pairs that we can safely evaluate under the current policy~$\pi_n$ without leaving the region of attraction, see~\cref{fig:example_set} (top, white shaded).

What remains is to define a strategy for collecting data points within~$\SafeSet_\ndata$ to effectively decrease model uncertainty. We specifically focus on the high-level requirements for any exploration scheme without committing to a specific method. In practice, any (model-based) exploration strategy that aims to decrease model uncertainty by driving the system to specific states may be used. Safety can be ensured by picking actions according to~$\pi_\ndata$ whenever the exploration strategy reaches the boundary of the safe region~$\mathcal{V}(c_n)$; that is, when~$u_n(\xb, \ub) > c_n$. This way, we can use~$\pi_\ndata$ as a backup policy for exploration.

The high-level goal of the exploration strategy is to shrink the confidence intervals at state-action pairs~$\SafeSet_\ndata$ in order to expand the safe region. Specifically, the exploration strategy should aim to visit state-action pairs in~$\SafeSet_\ndata$ at which we are the most uncertain about the dynamics; that is, where the confidence interval is the largest:
\begin{equation}
  (\xb_\ndata, \ub_\ndata) = \argmax_{(\xb, \ub) \in \SafeSet_\ndata}  u_\ndata(\xb, \ub) - l_\ndata (\xb, \ub) . 
  \label{eq:sampling_criterion}
\end{equation}
As we keep collecting data points according to~\cref{eq:sampling_criterion}, we decrease the uncertainty about the dynamics for different actions throughout the region of attraction and adapt the policy, until eventually we have gathered enough information in order to expand it. While~\cref{eq:sampling_criterion} implicitly assumes that any state within~$\mathcal{V}(c_n)$ can be reached by the exploration policy, it achieves the high-level goal of any exploration algorithm that aims to reduce model uncertainty. In practice, any safe exploration scheme is limited by unreachable parts of the state space.

We compare the active learning scheme in~\cref{eq:sampling_criterion} to an oracle baseline that starts from the same initial safe set~$\SafeSet_0$ and knows~$v(f(\xb,\ub))$ up to $\epsilon$ accuracy within the safe set. The oracle also uses knowledge about the Lipschitz constants and the optimal policy in~$\Pi_L$ at each iteration. We denote the set that this baseline manages to determine as safe with~$\Rbar_\epsilon(\SafeSet_0)$ and provide a detailed definition in~\cref{sec:app_baseline}.
\begin{restatable}{theorem}{maintheorem}
  Assume $\sigma$-sub-Gaussian measurement noise and that the model error~$g(\cdot)$ in~\cref{eq:dynamic_system} has RKHS norm smaller than~$B_g$.
  Under the assumptions of~\cref{thm:gp_region_of_attraction}, with~${\beta_\ndata = B_g + 4 \sigma \sqrt{\gamma_\ndata + 1 + \ln(1 / \delta)} }$, and with measurements collected according to~\cref{eq:sampling_criterion}, let $\ndata^*$ be the smallest positive integer so that
  $
    \frac{\ndata^*}{\beta_{\ndata^*}^2 \gamma_{\ndata^*}} \geq \frac{C q (| \Rbar(\SafeSet_0)| + 1 )}{L_v^2 \epsilon^2}
  $
  where $C = 8 / \log(1 + \sigma^{-2})$. Let $\mathcal{R}_{\pi}$ be the true region of attraction of~\cref{eq:dynamic_system} under a policy~$\pi$. For any $\epsilon > 0$, and $\delta \in (0, 1)$, the following holds jointly with probability at least $(1 - \delta)$ for all $\ndata > 0$:

  \begin{enumerate*}[label= (\roman*)]
    \item  ~${\mathcal{V}(c_\ndata) \subseteq \mathcal{R}_{\pi_\ndata}}$
\label{thm:main_rlev_D_contained_in_roa} \hspace{2.5em}
    \item  ~$f(\xb, \ub) \in \mathcal{R}_{\pi_\ndata}\, \forall (\xb, \ub) \in \SafeSet_\ndata$.
\label{thm:main_x_next_contained_in_roa} \hspace{2.5em}
    \item ~$\Rbar_\epsilon(\SafeSet_0) \subseteq \SafeSet_\ndata \subseteq \Rbar_0(\SafeSet_0)$.
\label{thm:main_achieve_baseline}
\end{enumerate*}
\label{thm:exploration_guarantees}
\end{restatable}
\cref{thm:exploration_guarantees} states that, when selecting data points according to~\cref{eq:sampling_criterion}, the estimated region of attraction~$\mathcal{V}(c_n)$ is~\cref{thm:main_rlev_D_contained_in_roa} contained in the true region of attraction under the current policy and~\cref{thm:main_x_next_contained_in_roa} selected data points do not cause the system to leave the region of attraction. This means that any exploration method that considers the safety constraint~\cref{eq:gp_safe_set} is able to safely learn about the system without leaving the region of attraction. The last part of~\cref{thm:exploration_guarantees},~\cref{thm:main_achieve_baseline}, states that after a finite number of data points~$\ndata^*$ we achieve at least the exploration performance of the oracle baseline, while we do not classify unsafe state-action pairs as safe. This means that the algorithm explores the largest region of attraction possible for a given Lyapunov function with residual uncertaint about~$v(f(\cdot))$ smaller than~$\epsilon$. Details of the comparison baseline are given in the appendix. In practice, this means that any exploration method that manages to reduce the maximal uncertainty about the dynamics within~$\SafeSet_\ndata$ is able to expand the region of attraction.

An example run of repeatedly evaluating~\cref{eq:sampling_criterion} for a one-dimensional state-space is shown in~\cref{fig:example_set}. It can be seen that, by only selecting data points within the current estimate of the region of attraction, the algorithm can efficiently optimize the policy and expand the safe region over time.


\section{Practical Implementation and Experiments}

\begin{algorithm}[t]
  \caption{\textsc{SafeLyapunovLearning}}
  \begin{algorithmic}[1]
    \setcounter{ALC@unique}{0}
    \STATE{} \textbf{Input:} Initial safe policy $\pi_0$, dynamics model~$\mathcal{GP}(\mu(\ab), k(\ab, \ab'))$  \\
    \FORALL{$\ndata = 1, \dots$}
      \STATE{} Compute policy~$\pi_n$ via SGD on \cref{eq:lagrangian} \\
      \STATE{} $c_n = \argmax_c c, \textnormal{such~that~} \forall \xb \in \mathcal{V}(c_n) \cap \mathcal{X}_\tau \colon u_n(\xb, \pi_\ndata(\xb)) -v(\xb) < -L_{\Delta v}\tau$  \\
      \STATE{} $\SafeSet_\ndata = \{ (\xb, \ub) \in \mathcal{V}(c_\ndata) \times \mathcal{U}_\tau \,|\, u_n(\xb, \ub) \leq c_\ndata \} $ \label{alg:safe_learning_Sn}
      \STATE{} Select $(\xb_\ndata, \ub_\ndata)$ within $\SafeSet_\ndata$  using~\cref{eq:sampling_criterion} and drive system there with backup policy~$\pi_n$ \\
      \STATE{} Update GP with measurements $f(\xb_\ndata, \ub_\ndata) + \epsilon_\ndata$\\
    \ENDFOR{}
  \end{algorithmic}
\label{alg:safe_learning}
\end{algorithm}

In the previous section, we have given strong theoretical results on safety and exploration for an idealized algorithm that can solve~\cref{eq:policy_update}. In this section, we provide a practical variant of the theoretical algorithm in the previous section. In particular, while we retain safety guarantees, we sacrifice exploration guarantees to obtain a more practical algorithm. This is summarized in~\cref{alg:safe_learning}.

The policy optimization problem in~\cref{eq:policy_update} is intractable to solve and only considers safety, rather than a performance metric. We propose to use an approximate policy update that that maximizes approximate performance while providing stability guarantees. It proceeds by optimizing the policy first and then computes the region of attraction~$\mathcal{V}(c_n)$ for the new, fixed policy. This does not impact safety, since data is still only collected inside the region of attraction. Moreover, should the optimization fail and the region of attraction decrease, one can always revert to the previous policy, which is guaranteed to be safe.

In our experiments, we use approximate dynamic programming~\citep{Powell2007Approximate} to capture the performance of the policy. Given a policy~$\pi_\theta$ with parameters~$\theta$, we compute an estimate of the cost-to-go~$J_{\pi_\theta}(\cdot)$ for the mean dynamics~$\mu_{\ndata}$ based on the cost $r(\xb, \ub) \geq 0$. At each state,~$J_{\pi_\theta}(\xb)$ is the sum of $\gamma$-discounted rewards encountered when following the policy~$\pi_\theta$. The goal is to adapt the parameters of the policy for minimum cost as measured by~$J_{\pi_\theta}$, while ensuring that the safety constraint on the worst-case decrease on the Lyapunov function in~\cref{thm:gp_region_of_attraction} is not violated. A Lagrangian formulation to this constrained optimization problem is
\begin{equation}
  \pi_n = \argmin_{\pi_\theta \in \Pi_L}
  \sum_{\xb \in \mathcal{X}_\tau}
  r(\xb, \pi_\theta(\xb)) + \gamma J_{\pi_\theta}(\mu_{\ndata - 1}(\xb, \pi_\theta(\xb))
  + \lambda \Big(u_n(\xb, \pi_\theta(\xb)) - v(\xb) + L_{\Delta v} \tau \Big),
  \label{eq:lagrangian}
\end{equation}
where the first term measures long-term cost to go and~$\lambda \geq 0$ is a Lagrange multiplier for the safety constraint from~\cref{thm:gp_region_of_attraction}. In our experiments, we use the value function as a Lyapunov function candidate,~$v=J$ with~$r(\cdot,\cdot)\geq0$, and set~$\lambda=1$. In this case,~\cref{eq:lagrangian} corresponds to an high-probability upper bound on the cost-to-go given the uncertainty in the dynamics. This is similar to worst-case performance formulations found in robust MDPs~\cite{Tamar2014Scaling,Wiesemann2012Robust}, which consider worst-case value functions given parametric uncertainty in MDP transition model. Moreover, since~$L_{\Delta v}$ depends on the Lipschitz constant of the policy, this simultaneously serves as a regularizer on the parameters~$\theta$.

To verify safety, we use the GP confidence intervals~$l_n$ and $u_n$ directly, as in~\cref{eq:gp_dec_set_nice}. We also use confidence to compute~$\SafeSet_\ndata$ for the active learning scheme, see~\cref{alg:safe_learning}, Line~\ref{alg:safe_learning_Sn}. In practice, we do not need to compute the entire set~$\SafeSet_\ndata$ to solve~\cref{eq:policy_update}, but can use a global optimization method or even a random sampling scheme within~$\mathcal{V}(c_n)$ to find suitable state-actions.
Moreover, measurements for actions that are far away from the current policy are unlikely to expand~$\mathcal{V}(c_n)$, see~\cref{fig:example_set_3}.
As we optimize~\cref{eq:lagrangian} via gradient descent, the policy changes only locally. Thus, we can achieve better data-efficiency by restricting the exploratory actions~$\ub$ with $(\xb, \ub) \in \mathcal{S}_\ndata$ to be close to~$\pi_n$,~$\ub \in [\pi_\ndata(\xb) - \bar{u}, \pi_\ndata(\xb) + \bar{u}]$ for some constant~$\bar{u}$.

Computing the region of attraction by verifying the stability condition on a discretized domain suffers from the curse of dimensionality. However, it is not necessary to update policies in real time. In particular, since any policy that is returned by the algorithm is provably safe within some level set, any of these policies can be used safely for an arbitrary number of time steps. To scale this method to higher-dimensional system, one would have to consider an adaptive discretization for the verification as in~\cite{Bobiti2016Sampling}.

\label{sec:optimization}
\begin{figure*}[t]
  \subfigure[Estimated safe set.\label{fig:levelset}]{\includegraphics[scale=1]{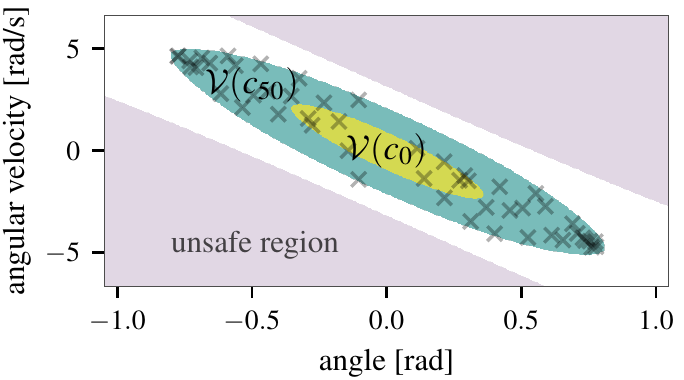}}
  \hfill
  \subfigure[State trajectory (lower is better). \label{fig:performance}]{\includegraphics[scale=1]{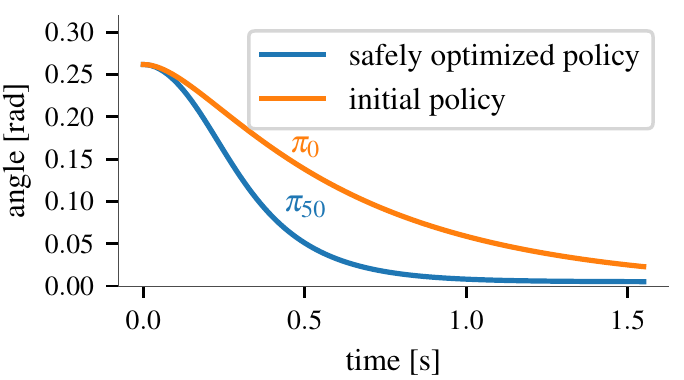}}
  \caption{Optimization results for an inverted pendulum. \cref{fig:levelset} shows the initial safe set (yellow) under the policy~$\pi_0$, while the green region represents the estimated region of attraction under the optimized neural network policy. It is contained within the true region of attraction (white). \cref{fig:performance} shows the improved performance of the safely learned policy over the policy for the prior model.
  }
\label{fig:experiments}
\end{figure*}

\paragraph{Experiments}
\label{sec:experiments}
\label{sec:practical}
A Python implementation of~\cref{alg:safe_learning} and the experiments based on TensorFlow~\cite{Abadi2016TensorFlow} and GPflow~\cite{Matthews2017GPflow} is available at~\url{https://github.com/befelix/safe_learning}.

We verify our approach on an inverted pendulum benchmark problem. The true, continuous-time dynamics are given by
$
m l^2 \ddot{\psi} =  g m l \sin(\psi) - \lambda \dot{\psi} + u
$, where~$\psi$ is the angle, $m$ the mass, $g$ the gravitational constant, and~$u$ the torque applied to the pendulum. The control torque is limited, so that the pendulum necessarily falls down beyond a certain angle. We use a GP model for the \textit{discrete-time} dynamics, where the mean dynamics are given by a linearized and discretized model of the true dynamics that considers a wrong, lower mass and neglects friction. As a result, the optimal policy for the mean dynamics does not perform well and has a small region of attraction as it underactuates the system. We use a combination of linear and Mat\'ern kernels in order to capture the model errors that result from parameter and integration errors.

For the policy, we use a neural network with two hidden layers and 32 neurons with ReLU activations each. We compute a conservative estimate of the Lipschitz constant as in~\citep{Szegedy2014Intriguing}. We use standard approximate dynamic programming with a quadratic, normalized cost~$r(\xb,\ub) = \xb^\T \mathbf{Q} \xb + \ub^\T \mathbf{R} \ub$, where~$\mathbf{Q}$ and~$\mathbf{R}$ are positive-definite, to compute the cost-to-go~$J_{\pi_\theta}$. Specifically, we use a piecewise-linear triangulation of the state-space as to approximate~$J_{\pi_\theta}$, see~\cite{Davies1996Multidimensional}.
This allows us to quickly verify the assumptions that we made about the Lyapunov function in~\cref{sec:problem} using a graph search.
In practice, one may use other function approximators. We optimize the policy via stochastic gradient descent on~\cref{eq:lagrangian}.

The theoretical confidence intervals for the GP model are conservative. To enable more data-efficient learning, we fix~$\beta_n = 2$. This corresponds to a high-probability decrease condition per-state, rather than jointly over the state space. Moreover, we use local Lipschitz constants of the Lyapunov function rather than the global one. While this does not affect guarantees, it greatly speeds up exploration.

For the initial policy, we use approximate dynamic programming to compute the optimal policy for the prior mean dynamics. This policy is unstable for large deviations from the initial state and has poor performance, as shown in~\cref{fig:performance}. Under this initial, suboptimal policy, the system is stable within a small region of the state-space~\cref{fig:levelset}.
Starting from this initial safe set, the algorithm proceeds to collect safe data points and improve the policy. As the uncertainty about the dynamics decreases, the policy improves and the estimated region of attraction increases. The region of attraction after 50 data points is shown in~\cref{fig:levelset}. The resulting set~$\mathcal{V}(c_n)$ is contained within the true safe region of the optimized policy~$\pi_\ndata$. At the same time, the control performance improves drastically relative to the initial policy, as can be seen in~\cref{fig:performance}. Overall, the approach enables safe learning about dynamic systems, as all data points collected during learning are safely collected under the current policy.


\section{Conclusion}
\label{sec:conclusion}

We have shown how classical reinforcement learning can be combined with safety constraints in terms of stability. Specifically, we showed how to safely optimize policies and give stability certificates based on statistical models of the dynamics. Moreover, we provided theoretical safety and exploration guarantees for an algorithm that can drive the system to desired state-action pairs during learning. We believe that our results present an important first step towards safe reinforcement learning algorithms that are applicable to real-world problems.

\subsubsection*{Acknowledgments}

This research was supported by SNSF grant {200020\_159557}, the Max Planck ETH Center for Learning Systems, NSERC grant {RGPIN-2014-04634}, and the Ontario Early Researcher Award.
%

\bibliographystyle{unsrt}
\bibliography{root.bib}

\appendix

\section{Proofs}
\label{sec:proofs}

\subsection{Stability verification}
\label{sec:proofs_stability}

\begin{restatable}{lemma}{deltaverrorbound}
  Using~\cref{as:lipschitz_continuity,as:f_confidence_interval}, let~$\mathcal{X}_\tau$ be a discretization of $\mathcal{X}$ such that ${\| \xb - {[\xb]}_\tau \|_1 \leq \tau}$ for all ${\xb \in \mathcal{X}}$. Then, for all $\xb \in \mathcal{X}$, we have with probability at least $1-\delta$ that
  \begin{equation}
    \big| v(\mu_{n-1}({[\ab]}_\tau)) - v({[\xb]}_\tau) - \big( v(f(\ab)) - v(\xb) \big) \big|
    \leq  L_v \beta_\ndata \sigma_{n-1}({[\ab]}_\tau) + (L_v L_f (L_\pi + 1) + L_v ) \tau,
  \end{equation}
  where $\ab = (\xb, \pi(\xb))$ and ${[\ab]}_\tau = ({[\xb]}_\tau, \pi({[\xb]}_\tau))$.
\label{lem:v_decrease_confidence}
\end{restatable}
\begin{proof}
Let $\ab = (\xb, \pi(\xb))$, ${[\ab]}_\tau = ({[\xb]}_\tau, \pi({[\xb]}_\tau))$, and $\mu = \mu_{n-1}$, $\sigma = \sigma_{n-1}$. Then we have that

\begin{align*}
  & \big| v(\mu({[\ab]}_\tau)) - v({[\xb]}_\tau) - \big( v(f(\ab)) - v(\xb) \big) \big|, \\
  ={}& \big| v(\mu({[\ab]}_\tau)) - v({[\xb]}_\tau) - v(f(\ab)) + v(\xb) \big|, \\
  ={}&  \big|v(\mu({[\ab]}_\tau)) - v(f({[\ab]}_\tau))
  + v(f({[\ab]}_\tau)) - v(f(\ab))
  + v(\xb)- v({[\xb]}_\tau) \big|, \\
  \leq{}&  \big|v(\mu({[\ab]}_\tau)) - v(f({[\ab]}_\tau))\big|
  + \big| v(f({[\ab]}_\tau)) - v(f(\ab))\big|
  + \big| v(\xb)- v({[\xb]}_\tau) \big|, \\
  \leq{}&  L_v \| \mu({[\ab]}_\tau) - f({[\ab]}_\tau) \|_1
  + L_v \|f({[\ab]}_\tau) - f(\ab) \|_1
  + L_v \|\xb - {[\xb]}_\tau \|_1, \\
  \leq{}&  L_v \beta_\ndata \sigma({[\ab]}_\tau) + L_v L_f \|{[\ab]}_\tau - \ab \|_1 + L_v \|\xb - {[\xb]}_\tau \|_1 ,
\end{align*}
where the last three inequalities follow from~\cref{as:lipschitz_continuity,as:f_confidence_interval} to last inequality follows from~\cref{lem:f_confidence_interval}. The result holds with probability at least $1-\delta$. By definition of the discretization and the policy class~$\Pi_L$ we have on each grid cell that
\begin{align*}
  \|\ab - {[\ab]}_\tau \|_1
  &= \| \xb - {[\xb]}_\tau \|_1 + \|\pi(\xb) - \pi({[\xb]}_\tau) \|_1, \\
  &\leq \tau + L_\pi \| \xb - {[\xb]}_\tau \|_1, \\
  & \leq (L_\pi + 1) \tau ,
\end{align*}
where the equality in the first step follows from the definition of the 1-norm. Plugging this into the previous bound yields
\begin{align*}
  & \big| v(\mu({[\ab]}_\tau)) - v({[\xb]}_\tau) - \big( v(f(\ab)) - v(\xb) \big) \big|
  \leq L_v \beta_\ndata \sigma({[\ab]}_\tau) + \left( L_v L_f (1 + L_\pi) + L_v \right) \tau , \\
\end{align*}
which completes the proof.
\end{proof}

\begin{lemma}
  $v(f(\xb, \ub)) \in \mathcal{Q}_\ndata$ holds for all $\xb \in \mathcal{X}$, $\ub \in \mathcal{U}$, and $\ndata > 0$ with probability at least $(1 - \delta)$.
\label{lem:Q_bounds_accurate}
\end{lemma}
\begin{proof}
  The proof is analogous to~\cref{lem:v_decrease_confidence} and follows from~\cref{as:lipschitz_continuity,as:f_confidence_interval}.
\end{proof}

\begin{corollary}
  $v(f(\xb, \ub)) \in \mathcal{C}_\ndata$ holds for all $\xb \in \mathcal{X}$, $\ub \in \mathcal{U}$, and $\ndata > 0$ with probability at least $(1 - \delta)$.
\label{cor:vf_contained_in_C}
\end{corollary}
\begin{proof}
  Direct consequence of the fact that~\cref{lem:Q_bounds_accurate} holds jointly for all $\ndata > 0$ with probability at least $1-\delta$.
\end{proof}

\cref{lem:v_decrease_confidence} show that the decrease on the Lyapunov function on the discrete grid~$\mathcal{X}_\tau$ is close to that on the continuous domain~$\mathcal{X}$. Given these confidence intervals, we can now establish the region of attraction using~\cref{thm:region_of_attraction}:

\gproa*
\begin{proof}
  Using~\cref{lem:v_decrease_confidence} it holds that $v(f(\xb, \pi(\xb)) - v(\xb) < 0$ for all continuous states $\xb \in \mathcal{V}(c)$ with probability at least $1-\delta$, since all discrete states $\xb_\tau \in \mathcal{V}(c) \cap \mathcal{X}$ fulfill~\cref{eq:safety_constraint}. Thus we can use~\cref{thm:lyapunov_stability} to conclude that $\mathcal{V}(c)$ is a region of attraction for~\cref{eq:dynamic_system}.
\end{proof}

\roanice*
\begin{proof}
  Following the definition of~$\DecSet_\ndata$ in~\cref{eq:gp_dec_set_nice}, it is clear from the constraint in the optimization problem~\cref{eq:policy_update} that for all~$\xb \in \DecSet_\ndata$ it holds that~$(\xb, \pi_\ndata(\xb)) \in \DecSet_\ndata$ or, equivalently that $u_\ndata(\xb, \pi(\xb)) - v(\xb) < -L_{\Delta v} \tau$, see~\cref{eq:gp_dec_set_nice}. The result~$\mathcal{V}(c_n) \subseteq \mathcal{R}_{\pi_\ndata}$ then follows from~\cref{thm:gp_region_of_attraction}.
\end{proof}

Note that the initialization of the confidence intervals~$\mathcal{Q}_0$ ensures that the decrease condition is always fulfilled for the initial policy.


\subsection{Gaussian process model}
\label{sec:gaussian_process_theory}

One particular assumption that satisfies both the Lipschitz continuity and allows us to use GPs as a model of the dynamics is that the model errors~$g(\xb, \ub)$ live in some reproducing kernel Hilbert space (RKHS,~\citep{Christmann2008Support}) corresponding to a differentiable kernel~$k$ and have RKHS norm smaller than $B_g$~\citep{Srinivas2012Gaussian}. In our theoretical analysis, we use this assumption to prove exploration guarantees.

A $\mathcal{GP}(\mu(\ab), k(\ab, \ab'))$ is a distribution over well-behaved, smooth functions~${f\colon \mathcal{X} \times \mathcal{U} \to \mathbb{R}}$ (see~\cref{rem:multi_dim_gp} for the vector-case, $\mathbb{R}^q$) that is parameterized by a mean function~$\mu$ and a covariance function (kernel)~$k$, which encodes assumptions about the functions~\citep{Rasmussen2006Gaussian}. In our case, the mean is given by the prior model~$h$, while the kernel corresponds to the one in the RKHS\@. Given noisy measurements of the dynamics, $\hat{f}(\ab) = f(\ab) + \epsilon$ with~$\ab = (\xb, \ub)$ at locations $A_\ndata = \{ \ab_1, \dots, \ab_\ndata \}$, corrupted by independent, Gaussian noise $\epsilon \sim \mathcal{N}(0, \sigma^2)$ (we relax the Gaussian noise assumption in our analysis), the posterior is a GP distribution again with mean,
$
  \mu_\ndata(\ab) = \mb{k}_\ndata(\ab)^\T (\mb{K}_\ndata + \sigma^2 \mb{I})^{-1} \mb{y}_\ndata$,
covariance
  $k_\ndata(\ab, \ab') = k(\ab, \ab') -\mb{k}_\ndata(\ab)^\T (\mb{K}_\ndata + \sigma^2 \mb{I})^{-1} \mb{k}_\ndata(\ab')$,
  and variance $\sigma^2_\ndata(\ab) = k_\ndata(\ab, \ab)$.
The vector ${
	\mb{y}_\ndata = [
	\begin{matrix}
		\hat{f}(\ab_1) - h(\ab_1),\dots,\hat{f}(\ab_\ndata) - h(\ab_\ndata)
	\end{matrix}
	]^\T
}$
contains observed, noisy deviations from the mean,
${\mb{k}_\ndata(\ab) =
[ \begin{matrix}
	k(\ab,\ab_1),\dots,k(\ab,\ab_\ndata)
\end{matrix} ]}$
contains the covariances between the test input~$\ab$ and the data points in~$\mathcal{D}_\ndata$, ${\mb{K}_\ndata \in \mathbb{R}^{n \times n}}$ has entries ${ {[\mb{K}_\ndata]}_{(i,j)} = k(\ab_i, \ab_j) }$, and $\mb{I}$ is the identity matrix.

\begin{remark}
\label{rem:multi_dim_gp}
In the case of multiple output dimensions (${\nstate > 1}$), we consider a function with one-dimensional output $f'(\xb, \ub, i) \colon \mathcal{X} \times \mathcal{U} \times \mathcal{I} \to \mathbb{R}$, with the output dimension indexed by $i \in \mathcal{I} = \{ 1, \dots, \nstate \} $. This allows us to use the standard definitions of the RKHS norm and GP model. In this case, we define the GP posterior distribution as
$\mu_\ndata(\ab) = [ \mu_\ndata(\ab, 1), \dots, \mu_\ndata(\ab, \nstate) ]^\T$ and
$\sigma_\ndata(\ab) = \sum_{1 \leq i \leq \nstate}\sigma_\ndata(\ab, i)$,
where the unusual definition of the standard deviation is used in~\cref{lem:f_confidence_interval}.
\end{remark}

Given the previous assumptions, it follows from \citep[Lemma 2]{Berkenkamp2016Lyapunov} that the dynamics in~\cref{eq:dynamic_system} are Lipschitz continuous with Lipschitz constant ${L_f = L_h + L_g}$, where $L_g$ depends on the properties (smoothness) of the kernel.

Moreover, we can construct high-probability confidence intervals on the dynamics in~\cref{eq:dynamic_system} that fulfill~\cref{as:f_confidence_interval} using the GP model.
\begin{restatable}{lemma}{dynamicsconfidence}
\label{lem:f_confidence_interval} (\citep[Theorem 6]{Srinivas2012Gaussian})
Assume $\sigma$-sub-Gaussian noise and that the model error~$g(\cdot)$ in~\cref{eq:dynamic_system} has RKHS norm bounded by~$B_g$.
Choose ${\beta_\ndata = B_g + 4 \sigma \sqrt{\gamma_\ndata + 1 + \ln(1 / \delta)} }$. Then,  with probability at least~${1 - \delta}$, ${\delta \in (0, 1)}$, for all~${\ndata \geq 1}$, ${\xb \in \mathcal{X}}$, and~${\ub \in \mathcal{U}}$ it holds that
${
  \| f(\xb, \ub) - \mu_{\ndata-1}(\xb, \ub) \|_1 \leq \beta_\ndata \sigma_{\ndata-1}(\xb, \ub).
}$
\label{eq:f_confidence_interval}
\end{restatable}
%
\begin{proof}
  From~\citep[Theorem 2]{Chowdhury2017Kernelized} it follows that $|f(\xb, \ub, i) - \mu_{n-1}(\xb, \ub, i)| \leq \beta_{\ndata} \sigma_\ndata(\xb, \ub, i)$ holds with probability at least $1-\delta$ for all~$1 \leq i \leq \nstate$. Following~\cref{rem:multi_dim_gp}, we can model the multi-output function as a single-output function over an extended parameter space. Thus the result directly transfers by definition of the one norm and our definition of~$\sigma_\ndata$ for multiple output dimensions in~\cref{rem:multi_dim_gp}. Note that by iteration~$\ndata$ we have obtained $\ndata \nstate$ measurements in the information capacity~$\gamma_\ndata$.
\end{proof}
That is, the true dynamics are contained within the GP posterior confidence intervals with high probability. The bound depends on the information capacity,
\begin{equation}
  \gamma_\ndata = \max_{A \subset \mathcal{X} \times \mathcal{U} \times \mathcal{I} \colon |A| = \ndata q}  \mathrm{I}(\mb{y}_A; \mb{f}_A),
  \label{eq:information_capacity}
\end{equation}
which is the maximum mutual information that could be gained about the dynamics~$f$ from samples. The information capacity has a sublinear dependence on $\ndata (\neq \ti)$ for many commonly used kernels such as the linear, squared exponential, and Mat\'ern kernels and it can be efficiently and accurately approximated~\citep{Srinivas2012Gaussian}. Note that we explicitly account for the $q$ measurements that we get for each of the $\nstate$ states in~\cref{eq:information_capacity}.

\begin{remark}
  The GP model assumes Gaussian noise, while~\cref{lem:f_confidence_interval} considers~$\sigma$-sub-Gaussian noise. Moreover, we consider functions with bounded RKHS norm, rather than samples from a GP. \cref{lem:f_confidence_interval} thus states that even though we make different assumptions than the model, the confidence intervals are conservative enough to capture the true function with high probability.
\end{remark}

\subsection{Safe exploration}
\label{sec:proofs_exploration}

\begin{remark}
  In the following we assume that~$\DecSet_\ndata$ and~$\SafeSet_\ndata$ are defined as in~\cref{eq:gp_dec_set} and~\cref{eq:gp_safe_set}.
\end{remark}

\paragraph{Baseline}
\label{sec:app_baseline}
As a baseline, we consider a class of algorithms that know about the Lipschitz continuity properties of $v$, $f$, and $\pi$. In addition, we can learn about $v(f(\xb, \ub))$ up to some arbitrary statistical accuracy~$\epsilon$ by visiting state $\xb$ and obtaining a measurement for the next state after applying action $\ub$, but face the safety restrictions defined in~\cref{sec:problem}.
Suppose we are given a set~$\SafeSet$ of state-action pairs about which we can learn safely. Specifically, this means that we have a policy such that, for any state-action pair $(\xb, \ub)$ in~$\SafeSet$, if we apply action $\ub$ in state $\xb$ and then apply actions according to the policy, the state  converges to the origin. Such a set can be constructed using the initial policy $\pi_0$ from~\cref{sec:problem} as~$\SafeSet_0 = \{(\xb,\pi_0(\xb)) \,|\, \xb \in \SafeSet_0^x \}$.

The goal of the algorithm is to expand this set of states that we can learn about safely. Thus, we need to estimate the region of attraction by certifying that state-action pairs achieve the $-L_{\Delta v} \tau$ decrease condition in~\cref{thm:gp_region_of_attraction} by learning about state-action pairs in~\SafeSet. We can then generalize the gained knowledge to unseen states by exploiting the Lipschitz continuity,
\begin{align}
  &\Rsafe(\SafeSet) = \SafeSet_0 \hspace{-0.2em}\cup\hspace{-0.2em}  \big\{ \ab \in \mathcal{X}_\tau \times \mathcal{U}_\tau  \, | \, \exists (\xb, \ub) \in \SafeSet \colon
  v(f(\xb, \ub))
  \hspace{-0.2em}-\hspace{-0.2em} v(\xb)
  \hspace{-0.2em}+\hspace{-0.2em} \epsilon
  \hspace{-0.2em}+\hspace{-0.2em} L_{\Delta v} \| \ab \hspace{-0.2em}-\hspace{-0.2em} (\xb, \ub) \|_1 \hspace{-0.2em}< \hspace{-0.2em}-L_{\Delta v} \tau \big\},
  \label{eq:Rdec}
\end{align}
where we use that we can learn $v(f(\xb, \ub))$ up to $\epsilon$ accuracy within~$\SafeSet$. We specifically include~$\SafeSet_0$ in this set, to allow for initial policies that are safe, but does not meet the strict decrease requirements of~\cref{thm:gp_region_of_attraction}. Given that all states in~$\Rsafe(\SafeSet)$ fulfill the requirements of~\cref{thm:gp_region_of_attraction}, we can estimate the corresponding region of attraction by committing to a control policy~${\pi \in \Pi_L}$ and estimating the largest safe level set of the Lyapunov function. With~$\mathcal{D}=\Rsafe(\SafeSet)$, the operator
\begin{equation}
  \Rlevel(\DecSet) = \mathcal{V} \big(
  \argmax c, \quad \mathrm{such~that~~} \exists \pi \in \Pi_L \colon
  \forall \xb \in \mathcal{V}(c) \cap \mathcal{X}_\tau,\, (\xb, \pi(\xb)) \in \DecSet
\big)
\label{eq:level_set_operator}
\end{equation}
encodes this operation. It optimizes over safe policies~$\pi \in \Pi_L$ to determine the largest level set, such that all state-action pairs $(\xb, \pi(\xb))$ at discrete states~$\xb$ in the level set~$\mathcal{V}(c) \cap \mathcal{X}_\tau$ fulfill the decrease condition of~\cref{thm:gp_region_of_attraction}. As a result,~$\Rlevel(\Rsafe(\SafeSet))$ is an estimate of the largest region of attraction given the $\epsilon$-accurate knowledge about state-action pairs in~$\SafeSet$. Based on this increased region of attraction, there are more states that we can safely learn about. Specifically, we again use the Lipschitz constant and statistical accuracy~$\epsilon$ to determine all states that map back into the region of attraction,
\begin{align}
  &\Reps(\SafeSet) = \SafeSet \hspace{-0.2em}\cup\hspace{-0.2em} \big\{ \ab' \in \Rlevel_\tau (\Rsafe(\SafeSet)) \times \mathcal{U}_\tau  \, | \, \exists \ab \in \SafeSet \colon
  v(f(\ab))
  \hspace{-0.2em} + \hspace{-0.2em} \epsilon
  \hspace{-0.2em} + \hspace{-0.2em} L_v L_f \|\ab \hspace{-0.2em}-\hspace{-0.2em} \ab'\|_1 \leq
  \hspace{-0.3cm}
  \max_{\xb \in \Rlevel(\Rsafe(\SafeSet))}
  \hspace{-0.6cm}
   v(\xb)  \big\},
  \label{math:Reps}
\end{align}
where~${ \Rlevel_\tau(\DecSet) = \Rlevel(\DecSet) \cap \mathcal{X}_\tau }$. Thus, $\Reps(\SafeSet) \supseteq \SafeSet$ contains state-action pairs that we can visit to learn about the system. Repeatedly applying this operator leads the largest set of state-action pairs that any safe algorithm with the same knowledge and restricted to policies in~$\Pi_L$ could hope to reach. Specifically, let ${ \Reps^0(\SafeSet) = \SafeSet }$ and ${\Reps^{i+1}(\SafeSet) = \Reps(\Reps^{i}(\SafeSet))}$. Then
$
  \Rbar_\epsilon(\SafeSet) = \lim_{i \to \infty} \Reps^i(\SafeSet)
$
is the set of all state-action pars on the discrete grid that any algorithm could hope to classify as safe without leaving this safe set. Moreover,~$\Rlevel(\Rbar_\epsilon(\SafeSet))$ is the largest corresponding region of attraction that any algorithm can classify as safe for the given Lyapunov function.

\paragraph{Proofs}
In the following we implicitly assume that the assumptions of~\cref{lem:f_confidence_interval} hold and that $\beta_\ndata$ is defined as specified within~\cref{lem:f_confidence_interval}. Moreover, for ease of notation we assume that $\SafeSet_0^x$ is a level set of the Lyapunov function $v(\cdot)$.

\begin{lemma}
  $\mathcal{V}(c_n) = \Rlevel(\DecSet_\ndata)$ and $c_n = \max_{\xb \in \Rlevel(\DecSet_\ndata)} v(\xb)$
  \label{lem:baseline_algo_transfer}
\end{lemma}
\begin{proof}
  Directly by definition, compare~\cref{eq:policy_update} and~\cref{eq:level_set_operator}.
\end{proof}

\begin{remark}
\cref{lem:baseline_algo_transfer} allows us to write the proofs entirely in terms of operators, rather than having to deal with explicit policies. In the following and in~\cref{alg:theory} we replace~$\mathcal{V}(c_n)$ and~$c_n$ according to~\cref{lem:baseline_algo_transfer}. This moves the definitions closer to the baseline and makes for an easier comparison.
\end{remark}

\begin{algorithm}[t]
  \begin{algorithmic}[1]
    \caption{Theoretical algorithm \label{alg:theory}}
    \setcounter{ALC@unique}{0}
    \STATE{} \textbf{Input:} Initial safe policy $\mathcal{S}_0$, dynamics model~$\mathcal{GP}(\mu(\ab), k(\ab, \ab'))$  \\
    \FORALL{$\ndata = 1, \dots$}
      \STATE{} $\DecSet_\ndata = \bigcup_{ (\xb, \ub) \in \SafeSet_{\ndata - 1}}  \big\{ \ab' \in \mathcal{X}_\tau \times \mathcal{U}_\tau \,|\,
      u_\ndata(\xb, \ub) - v(\xb) + L_{\Delta v} \| \ab' - (\xb, \ub) \|_1 < -L_{\Delta v} \tau \big\},$ \\
      \STATE{} $\pi_n, c_n = \argmax_{\pi \in \Pi_L, c \in \mathbb{R}_{>0}} c, \quad \textnormal{such~that~for all~} \xb \in \mathcal{V}(c) \cap \mathcal{X}_\tau \colon  (\xb, \pi(\xb)) \in \mathcal{D}_\ndata $ \\
      \STATE{} $\SafeSet_\ndata = \bigcup_{\ab \in \SafeSet_{\ndata - 1}} \big\{ \ab' \in \mathcal{V} ( c_n ) \cap \mathcal{X}_\tau \times \mathcal{U}_\tau \,|\,
      u_\ndata(\ab) + L_v L_f \| \ab - \ab' \|_1 \leq c_n \}$
      \STATE{} $\phantom{\SafeSet_\ndata} =  \bigcup_{\ab \in \SafeSet_{\ndata - 1}} \big\{ \ab' \in \Rlevel_\tau(\DecSet_\ndata)  \times \mathcal{U}_\tau \,|\,
      u_\ndata(\ab) + L_v L_f \| \ab - \ab' \|_1 \leq \max_{\xb \in \Rlevel(\DecSet_\ndata)} v(\xb) \}$
      \STATE{} $(\xb_\ndata, \ub_\ndata) = \argmax_{(\xb, \ub) \in \SafeSet_\ndata}  u_\ndata(\xb, \ub) - l_\ndata (\xb, \ub) $ \\
      \STATE{} $\SafeSet_\ndata = \{ \ab \in \mathcal{V}(c_\ndata) \times \mathcal{U}_\tau \,|\, u_n(\ab) \leq c_\ndata \} $
      \STATE{} Update GP with measurements $f(\xb_\ndata, \ub_\ndata) + \epsilon_\ndata$\\
    \ENDFOR{}
  \end{algorithmic}
\label{alg:safe_learning_theory}
\end{algorithm}

We roughly follow the proof strategy in~\citep{Sui2015Safe}, but deal with the additional complexity of having safe sets that are defined in a more difficult way (indirectly through the policy). This is non-trivial and the safe sets are carefully designed in order to ensure that the algorithm works for general nonlinear systems.

We start by listing some fundamental properties of the sets that we defined below.
\begin{lemma}
  It holds for all~$\ndata \geq 1$ that
  \begin{enumerate}[label={(\roman*)},leftmargin=*]
    \item $\forall \ab \in \mathcal{X}_\tau \times \mathcal{U}_\tau,\, u_{\ndata + 1}(\ab) \leq u_\ndata(\ab)$
\label{lem:prop_u_dec}
    \item $\forall \ab \in \mathcal{X}_\tau \times \mathcal{U}_\tau,\, l_{\ndata + 1}(\ab) \geq l_\ndata(\ab)$
\label{lem:prop_l_inc}
    \item $\mathcal{S} \subseteq \mathcal{R} \implies \Rlevel(\mathcal{S}) \subseteq \Rlevel(\mathcal{R})$
\label{lem:prop_Rlevel}
    \item $\mathcal{S} \subseteq \mathcal{R} \implies \Rsafe(\mathcal{S}) \subseteq \Rsafe(\mathcal{R})$
\label{lem:prop_Rsafe}
    \item $\mathcal{S} \subseteq \mathcal{R} \implies \Reps(\mathcal{S}) \subseteq \Reps(\mathcal{R})$
\label{lem:prop_Reps}
    \item $\mathcal{S} \subseteq \mathcal{R} \implies \Rbar_\epsilon(\mathcal{S}) \subseteq \Rbar_\epsilon(\mathcal{R})$
\label{lem:prop_Rbar}
    \item $\SafeSet_\ndata \supseteq \SafeSet_{\ndata - 1} \implies \DecSet_{\ndata+1} \supseteq \DecSet_{\ndata}$
\label{lem:prop_dec_simp2}
    \item $\DecSet_1 \supseteq \SafeSet_0$
\label{lem:prop_dec1_contains_s0}
    \item $\SafeSet_\ndata \supseteq \SafeSet_{\ndata - 1}$
\label{lem:prop_s_inc}
    \item $\DecSet_\ndata \supseteq \DecSet_{\ndata - 1}$
\label{lem:prop_d_inc}
  \end{enumerate}
\label{lem:properties_of_sets}
\end{lemma}
\begin{proof}
  \cref{lem:prop_u_dec,lem:prop_l_inc} follow directly form the definition of~$\mathcal{C}_\ndata$.

  \begin{enumerate}[label={(\roman*)},leftmargin=*]
    \addtocounter{enumi}{2}
    \item Let $\pi \in \Pi_L$ be a policy such that for some $c>0$ it holds for all $\xb \in \mathcal{V}(c) \cap \mathcal{X}_\tau$ that $(\xb, \pi(\xb)) \in \mathcal{S}$. Then we have that $(\xb, \pi(\xb)) \in \mathcal{R}$, since~$\mathcal{S} \subseteq \mathcal{R}$. Thus it follows that
    with
    \begin{equation}
      \begin{aligned}
      c_s = &\argmax c \quad \mathrm{s.t.~} \exists \pi \in \Pi_L \colon  \forall \xb \in \mathcal{V}(c) \cap \mathcal{X}_\tau,\, (\xb, \pi(\xb)) \in \mathcal{S}
      \end{aligned}
    \end{equation}
    and
    \begin{equation}
      \begin{aligned}
      c_r = &\argmax c \quad \mathrm{s.t.~} \exists \pi \in \Pi_L \colon  \forall \xb \in \mathcal{V}(c) \cap \mathcal{X}_\tau,\, (\xb, \pi(\xb)) \in \mathcal{R}
      \end{aligned}
    \end{equation}
    we have that $c_r \geq c_s$. This implies $\mathcal{V}(c_r) \supseteq \mathcal{V}(c_s)$. The result follows.
    \item Let $\ab \in \Rsafe(\mathcal{S})$. Then there exists $(\xb, \ub) \in \mathcal{S}$ such that $v(f(\xb, \ub)) - v(\xb) + \epsilon + L_{\Delta v} \| \ab - (\xb, \ub) \|_1 < -L_{\Delta v} \tau$. Since $\mathcal{S} \subseteq \mathcal{R}$ we have that $(\xb, \ub) \in \mathcal{R}$ as well and thus $\ab \in \Rsafe(\mathcal{R})$.
    \item $\mathcal{S} \subseteq \mathcal{R} \implies \Rlevel(\Rsafe(\mathcal{S})) \subseteq \Rlevel(\Rsafe(\mathcal{R}))$ due to~\cref{lem:prop_Rlevel,lem:prop_Rsafe}. Since $\ab' \in \Reps(\mathcal{S})$, there must exist an $\ab \in \mathcal{S}$ such that $v(f(\ab)) + \epsilon + L_v L_f \| \ab - \ab' \|_1 \leq \max_{\xb \in \Rlevel(\Rsafe(\mathcal{S}))} v(\xb)$. Since $\mathcal{S} \subseteq \mathcal{R}$ it follows that $\ab \in \mathcal{R}$. Moreover,
    \begin{equation}
      \max_{\xb \in \Rlevel(\Rsafe(\mathcal{S}))} v(\xb) \leq \max_{\xb \in \Rlevel(\Rsafe(\mathcal{R}))} v(\xb)
    \end{equation}
    follows from  $\Rlevel(\Rsafe(\mathcal{S})) \subseteq \Rlevel(\Rsafe(\mathcal{R}))$, so that we conclude that $\ab' \in \Reps(\mathcal{R})$.
    \item This follows directly by repeatedly applying the result of~\cref{lem:prop_Reps}.
    \item Let $\ab' \in \DecSet_{\ndata}$. Then $\exists (\xb, \ub) \in \SafeSet_{\ndata - 1} \colon u_\ndata(\xb, \ub) - v(\xb) + L_{\Delta v} \| \ab' - (\xb, \ub) \|_1 < -L_{\Delta v} \tau$. Since $\SafeSet_\ndata \supseteq \SafeSet_{\ndata - 1}$ it follows that $(\xb, \ub) \in \SafeSet_{\ndata}$ as well. Moreover, we have
    \begin{align*}
      & u_{\ndata + 1}(\xb, \ub) - v(\xb) + L_{\Delta v} \| \ab' - (\xb, \ub) \|_1 \\
      \leq{}& u_{\ndata}(\xb, \ub) - v(\xb) + L_{\Delta v} \| \ab' - (\xb, \ub) \|_1 < -L_{\Delta v} \tau
      \end{align*}
      since $u_{\ndata + 1}$ is non-increasing, see~\cref{lem:prop_u_dec}. Thus $\ab' \in \DecSet_{\ndata + 1}$.
    \item By definition of $\mathcal{C}_0$ we have for all $(\xb, \ub) \in \SafeSet_0$ that $u_0(\xb, \ub) < v(\xb) - L_{\Delta v} \tau$. Now we have that
    \begin{align*}
      &u_1(\xb, \ub) - v(\xb) + L_{\Delta v} \| (\xb, \ub) - (\xb, \ub) \|_1,  \\
      ={}& u_1(\xb, \ub) - v(\xb), \\
      \leq{}& u_0(\xb, \ub) - v(\xb) , \quad \text{by~\cref{lem:properties_of_sets} \cref{lem:prop_u_dec}} \\
      <{}& -L_{\Delta v} \tau ,
    \end{align*}
    which implies that $(\xb, \ub) \in \DecSet_1$.
    \item Proof by induction. We consider the base case, $\ab \in \SafeSet_0$, which implies that $\ab \in \DecSet_1$ by~\cref{lem:prop_dec1_contains_s0}. Moreover, since $\SafeSet_0^x$ is a level set of the Lyapunov function $v$ by assumption, we have that $\Rlevel(\SafeSet_0) = \SafeSet^x_0$. The previous statements together with~\cref{lem:prop_Rlevel} imply that $\ab \in \Rlevel_\tau(\DecSet_1) \times \mathcal{U}_\tau$, since~$\DecSet_1 \supseteq \SafeSet_0$ by~\cref{lem:prop_dec1_contains_s0}. Now, we have that
    \begin{align*}
      u_1(\ab) + L_v L_f \| \ab - \ab \|_1 &= u_1(\ab)  \overset{\cref{lem:prop_u_dec}}{\leq} u_0(\ab).
    \end{align*}
    Moreover, by definition of $\mathcal{C}_0$, we have for all $(\xb, \ub) \in \SafeSet_0$ that $$u_0(\xb, \ub) < v(\xb) - L_{\Delta v} \tau < v(\xb).$$ As a consequence,
    \begin{align}
    u_0(\xb, \ub) &\leq \max_{(\xb, \ub) \in \SafeSet_0} v(\xb), \\
    &= \max_{\xb \in \Rlevel(\SafeSet_0)} v(\xb), \\
    &\leq \max_{\xb \in \Rlevel(\DecSet_1)} v(\xb),
    \end{align}
    where the last inequality follows from~\cref{lem:prop_dec1_contains_s0,lem:prop_Rlevel}. Thus we have $\ab \in \SafeSet_1$.

    For the induction step, assume that for $\ndata \geq 2$ we have $\ab' \in \SafeSet_\ndata$ with $\SafeSet_\ndata \supseteq \SafeSet_{\ndata - 1}$. Now since $\ab' \in \SafeSet_\ndata$ we must have that $\ab' \in \Rlevel_\tau(\DecSet_{\ndata}) \times \mathcal{U}_\tau$. This implies that $\ab' \in \Rlevel_\tau(\mathcal{D}_{\ndata + 1}) \times \mathcal{U}_\tau$, due to \cref{lem:properties_of_sets}~\cref{lem:prop_dec_simp2,lem:prop_Rlevel} together with the induction assumption of $\SafeSet_\ndata \supseteq \SafeSet_{\ndata - 1}$. Moreover, there must exist a $\ab \in \SafeSet_{\ndata - 1} \subseteq \SafeSet_\ndata$ such that
    \begin{align}
      u_{\ndata + 1}(\ab)  + L_v L_f \| \ab - \ab' \|_1 ,
      &\leq u_\ndata (\ab)  + L_v L_f \| \ab - \ab' \|_1 , \\
      &\leq \max_{\xb \in \Rlevel(\DecSet_{\ndata})} v(\xb), \\
      &\leq \max_{\xb \in \Rlevel(\DecSet_{\ndata + 1})} v(\xb) ,
    \end{align}
    which in turn implies $\ab \in \SafeSet_{\ndata + 1}$. The last inequality follows from~\cref{lem:properties_of_sets}~\cref{lem:prop_Rlevel,lem:prop_dec_simp2} together with the induction assumption that $\SafeSet_\ndata \supseteq \SafeSet_{\ndata - 1}$.
    \item This is a direct consequence of~\cref{lem:prop_dec_simp2},~\cref{lem:prop_dec1_contains_s0}, and~\cref{lem:prop_s_inc}.

  \end{enumerate}
\end{proof}

Given these set properties, we first consider what happens if the safe set $\SafeSet_\ndata$ does not expand after collecting data points. We use these results later to conclude that the safe set must either expand or that the maximum level set is reached. We denote
by $$\ab_\ndata = (\xb_\ndata, \ub_\ndata)$$ the data point the is sampled according to~\cref{eq:sampling_criterion}.


\begin{lemma}
  For any $\ndata_1 \geq \ndata_0 \geq 1$, if $\SafeSet_{\ndata_1} = \SafeSet_{\ndata_0}$, then for any $\ndata$ such that $\ndata_0 \leq \ndata < \ndata_1$, it holds that
  \begin{equation}
    2 \beta_\ndata \sigma_{\ndata}(\ab_\ndata) \leq \sqrt{ \frac{C_1 q \beta^2_\ndata \gamma_\ndata}{\ndata - \ndata_0}},
  \end{equation}
  where $C_1 = 8 / \log(1 + \sigma^{-2})$.
\label{lem:f_uncertainty_bound}
\end{lemma}
\begin{proof}
  We modify the results for $q=1$ by~\citep{Srinivas2012Gaussian} to this lemma, but use the different definition for~$\beta_\ndata$ from~\cite{Chowdhury2017Kernelized}. Even though the goal of \citep[Lemma 5.4]{Srinivas2012Gaussian} is different from ours, we can still apply their reasoning to bound the amplitude of the confidence interval of the dynamics. In particular, in~\citep[Lemma 5.4]{Srinivas2012Gaussian}, we have $r_\ndata = 2 \beta_\ndata \sigma_{\ndata - 1}(\ab_\ndata)$ with $\ab_\ndata = (\xb_\ndata, \ub_\ndata)$ according to~\cref{eq:f_confidence_interval}. Then
  \begin{align}
    r_\ndata^2 &= 4 \beta^2_\ndata \sigma^2_{\ndata-1}(\ab_\ndata), \\
    &= 4 \beta^2_\ndata {\left( \sum_{i=1}^q \sigma_{\ndata-1}(\ab_\ndata, i) \right)}^2, \\
    &\leq 4 \beta^2_\ndata q \sum_{i=1}^q \sigma^2_{\ndata-1}(\ab_\ndata, i) \qquad \text{(Jensen's ineq.)}, \\
    &\leq 4 \beta^2_\ndata q \sigma^2 C_2 \sum_{i=1}^q \log(1 + \sigma^{-2} \sigma_{\ndata - 1}^2(\ab_\ndata, i)),
  \end{align}
  where $C_2 = \sigma^{-2} / \log(1 + \sigma^{-2})$. The result then follows analogously to~\citep[Lemma 5.4]{Srinivas2012Gaussian} by noting that
  \begin{equation}
      \sum_{j= 1}^{\ndata} r_j^2 \leq C_1 \beta^2_\ndata q \gamma_\ndata \qquad \forall \ndata \geq 1
  \end{equation}
  according to the definition of~$\gamma_\ndata$ in this paper and using the Cauchy-Schwarz inequality.
\end{proof}

The previous result allows us to bound the width of the confidence intervals:
\begin{corollary}
  For any $\ndata_1 \geq \ndata_0 \geq 1$, if $\SafeSet_{\ndata_1} = \SafeSet_{\ndata_0}$, then for any $\ndata$ such that $\ndata_0 \leq \ndata < \ndata_1$, it holds that
  \begin{equation}
    u_\ndata(\ab_\ndata) - l_\ndata(\ab_\ndata) \leq L_v \sqrt{ \frac{C_1 q \beta^2_\ndata \gamma_\ndata}{\ndata - \ndata_0}},
  \end{equation}
  where $C_1 = 8 / \log(1 + \sigma^{-2})$.
\label{cor:vf_uncertainty_bound}
\end{corollary}
\begin{proof}
  Direct consequence of~\cref{lem:f_uncertainty_bound} together with the definition of~$\mathcal{C}$ and $\mathcal{Q}$.
\end{proof}

\begin{corollary}
  For any ${\ndata \geq 1}$ with $C_1$ as defined in~\cref{lem:f_uncertainty_bound}, let $N_\ndata$ be the smallest integer satisfying $\frac{N_\ndata}{\beta^2_{\ndata + N_\ndata} \gamma{\ndata + N_\ndata}} \geq \frac{C_1 L_v^2 q}{\epsilon^2}$ and $\SafeSet_{\ndata + N_\ndata} = \SafeSet_{N_\ndata}$, then, for any $\ab \in \SafeSet_{\ndata + N_\ndata}$ it holds that
  \begin{equation}
    u_\ndata(\ab) - l_\ndata(\ab) \leq \epsilon.
  \end{equation}
\label{cor:Nstar_definition}
\end{corollary}

\begin{proof}
The result trivially follows from substituting $N_\ndata$ in the bound in \cref{cor:vf_uncertainty_bound}.
\end{proof}

\begin{lemma}
  For any ${\ndata \geq 1}$, if $\Rbar_\epsilon(\SafeSet_0) \setminus \SafeSet_\ndata \neq \emptyset$, then $\Reps(\SafeSet_\ndata) \setminus \SafeSet_\ndata \neq \emptyset$.
\label{lem:can_expand_if_not_done}
\end{lemma}
\begin{proof}
  As in~\citep[Lemma 6]{Sui2015Safe}. Assume, to the contrary, that $\Reps(\SafeSet_\ndata) \setminus \SafeSet_\ndata = \emptyset$. By definition~${\Reps(\SafeSet_\ndata) \supseteq \SafeSet_\ndata}$, therefore ${\Reps(\SafeSet_\ndata) = \SafeSet_\ndata}$. Iteratively applying $\Reps$ to both sides, we get in the limit $\Rbar_\epsilon (\SafeSet_\ndata) = \SafeSet_\ndata$. But then, by~\cref{lem:properties_of_sets},\cref{lem:prop_s_inc,lem:prop_Rbar}, we get
  \begin{equation}
    \Rbar_\epsilon(\SafeSet_0) \subseteq \Rbar_\epsilon(\SafeSet_\ndata) = \SafeSet_\ndata,
  \end{equation}
  which contradicts the assumption that $\Rbar_\epsilon(\SafeSet_0) \setminus \SafeSet_\ndata \neq \emptyset$.
\end{proof}

\begin{lemma}
  For any $\ndata \geq 1$, if $\Rbar_\epsilon(\SafeSet_0) \setminus \SafeSet_\ndata \neq \emptyset$, then the following holds with probability at least $1 - \delta$:
  \begin{equation}
    \SafeSet_{\ndata + N_\ndata} \supset \SafeSet_\ndata .
  \end{equation}
\end{lemma}

\begin{proof}
  By~\cref{lem:can_expand_if_not_done}, we have that~$\Reps(\SafeSet_\ndata) \setminus \SafeSet_\ndata \neq \emptyset$. By definition, this means that
  there exist $\ab \in \Reps(\SafeSet_\ndata) \setminus \SafeSet_\ndata$ and $\ab' \in \SafeSet_\ndata$ such that
  \begin{align}
    v(f(\ab')) + \epsilon + L_v L_f \|\ab - \ab' \|_1 \leq \max_{\xb \in \Rlevel(\Rsafe(\SafeSet_\ndata))} v(\xb)
    \label{eq:show_increase}
  \end{align}

  Now we assume, to the contrary, that $\SafeSet_{\ndata + N_\ndata} = \SafeSet_\ndata$ (the safe set cannot decrease due to~\cref{lem:properties_of_sets} \cref{lem:prop_s_inc}). This implies that $\ab \in \mathcal{X}_\tau \times \mathcal{U}_\tau \setminus \SafeSet_{\ndata + N_\ndata}$ and $\ab' \in \SafeSet_{\ndata + N_\ndata} = \SafeSet_{\ndata + N_\ndata - 1}$. Due to~\cref{cor:vf_uncertainty_bound}, it follows that
  \begin{align}
    & u_{\ndata + N_\ndata}(\ab') + L_v L_f \|\ab - \ab' \|_1 \\
    \leq{}& v(f(\ab')) + \epsilon  + L_v L_f \|\ab - \ab' \|_1 \\
    \leq{}& \max_{\xb \in \Rlevel(\Rsafe(\SafeSet_\ndata))} v(\xb)   && \text{by~\cref{eq:show_increase}} \\
    ={}& \max_{\xb \in \Rlevel(\Rsafe(\SafeSet_{\ndata + N_\ndata}))} v(\xb) && \text{by~\cref{lem:prop_Rlevel,lem:prop_Rsafe,lem:prop_s_inc}}
  \end{align}
  Thus, to conclude that~$\ab \in \SafeSet_{\ndata + N_\ndata}$ according to~\cref{eq:gp_safe_set}, we need to show that ${ \Rlevel(\DecSet_{\ndata + N_\ndata}) \supseteq \Rlevel(\Rsafe(\SafeSet_\ndata)) }$. To this end, we use~\cref{lem:properties_of_sets}~\cref{lem:prop_Rlevel} and show that ${\DecSet_{\ndata + N_\ndata} \supseteq \Rsafe(\mathcal{S}_{\ndata + N_\ndata})}$. Consider $(\xb, \ub) \in \Rsafe(\mathcal{S}_{\ndata + N_\ndata})$, we know that there exists a $(\xb',\ub') \in \SafeSet_{\ndata + N_\ndata} = \SafeSet_{\ndata + N_\ndata - 1}$ such that
  \begin{align}
-L_{\Delta v} \tau &> v(f(\xb', \ub')) - v(\xb') + \epsilon + L_{\Delta v} \| (\xb, \ub) - (\xb', \ub') \|_1 , \\
&\geq u_{\ndata + N_\ndata}(\xb', \ub') - v(\xb') + L_{\Delta v} \| (\xb, \ub) - (\xb', \ub') \|_1 ,
  \end{align}
  where the second inequality follows from \cref{cor:vf_uncertainty_bound}. This implies that~$(\xb, \ub) \in \DecSet_\ndata$ and thus $\DecSet_{\ndata + N_\ndata} \supseteq \Rsafe(\mathcal{S}_{\ndata + N_\ndata})$. This, in turn, implies that $\ab \in \SafeSet_{\ndata + N_\ndata}$, which is a contradiction.
\end{proof}

\begin{lemma}
  For any $\ndata \geq 0$, the following holds with probability at least $1-\delta$:
  \begin{equation}
    \SafeSet_\ndata \subseteq \Rbar_0(\SafeSet_0).
  \end{equation}
\label{lem:Sn_contained_in_Rbar0}
\end{lemma}
\begin{proof}
  Proof by induction. For the base case, $\ndata = 0$, we have $\SafeSet_0 \subseteq \Rbar_0(\SafeSet_0)$ by definition.

  For the induction step, assume that for some $\ndata \geq 1$, $\SafeSet_{\ndata - 1} \subseteq \Rbar_0 (\SafeSet_0)$. Let $\ab \in \SafeSet_\ndata$. Then, by definition, $\exists \ab' \in \SafeSet_{\ndata - 1}$ such that
  \begin{equation}
    u_\ndata(\ab') + L_v L_f \|\ab - \ab'\|_1 \leq \max_{\xb \in \Rlevel(\DecSet_{\ndata})} v(\xb),
  \end{equation}
  which, by~\cref{cor:vf_contained_in_C}, implies that
  \begin{equation}
    v(f(\ab')) + L_v L_f \|\ab - \ab'\|_1 \leq \max_{\xb \in \Rlevel(\DecSet_{\ndata})} v(\xb)
  \end{equation}
  Now since~$\ab' \in \Rbar_0(\SafeSet_0)$ by the induction hypothesis, in order to conclude that $\ab \in \Rbar_0(\SafeSet_0)$ we need to show that $\Rlevel(\DecSet_{\ndata}) \subseteq \Rlevel(\Rsafe(\Rbar(\SafeSet_0)))$ .

  Let $(\xb, \ub) \in \DecSet_{\ndata}$, then there exist $(\xb', \ab') \in \SafeSet_{\ndata -1}$ such that
  \begin{equation}
  u_{\ndata - 1}(\xb', \ub') - v(\xb')  + L_{\Delta v} \| (\xb, \ub) - (\xb', \ub') \|_1< -L_{\Delta v} \tau ,
  \end{equation}
  which, by~\cref{cor:vf_contained_in_C}, implies that
  \begin{equation}
    v(f(\xb', \ub')) - v(\xb') + L_{\Delta v} \| (\xb, \ub) - (\xb', \ub') \|_1 < -L_{\Delta v} \tau,
  \end{equation}
  which means that $(\xb, \ub) \in \Rsafe(\Rbar_0(\SafeSet_0))$ since $\SafeSet_{\ndata -1} \subseteq \Rbar_0(\SafeSet_0)$ and therefore~$(\xb', \ub') \in \Rbar_0(\SafeSet_0)$ holds by the induction hypothesis. We use~\cref{lem:prop_Rlevel} to conclude that $\Rlevel(\DecSet_{\ndata}) \subseteq \Rlevel(\Rsafe(\Rbar(\SafeSet_0)))$, which concludes the proof.
\end{proof}

\begin{lemma}
  Let $\ndata^*$ be the smallest integer, such that $\ndata^* \geq | \Rbar_0(\SafeSet_0) | N_{\ndata^*}$. Then, there exists $\ndata_0 \leq \ndata^*$ such that $\SafeSet_{\ndata_0 + N_{\ndata_0}} = \SafeSet_{\ndata_0}$ holds with probability~$1 - \delta$.
\label{lem:safeset_complete_at_some_point}
\end{lemma}
\begin{proof}
  By contradiction. Assume, to the contrary, that for all $\ndata \leq \ndata^* $, $\SafeSet_\ndata \subset \SafeSet_{\ndata + N_\ndata}$. From~\cref{lem:properties_of_sets} \cref{lem:prop_s_inc} we know that $\SafeSet_\ndata \subseteq \SafeSet_{\ndata + N_\ndata}$. Since $N_\ndata$ is increasing in $\ndata$, we have that $N_\ndata \leq N_{\ndata^*}$. Thus, we must have
  \begin{equation}
    \SafeSet_0 \subset \SafeSet_{N_{\ndata^*}} \subset \SafeSet_{2 N_{\ndata^*}} \cdots,
  \end{equation}
  so that for any $0\leq j \leq |\Rbar_0(\SafeSet_0)|$, it holds that $|S_{j T_{\ndata^*}} | > j$. In particular, for $j = | \Rbar_0(\SafeSet_0) |$, we get
  \begin{equation}
    | \SafeSet_{j N_{\ndata^*}} | > |\Rbar_0 (\SafeSet_0) |,
  \end{equation}
  which contradicts $\SafeSet_{j N_{\ndata^*}} \subseteq \Rbar_0(\SafeSet_0)$ from~\cref{lem:Sn_contained_in_Rbar0}.
\end{proof}

\begin{corollary}
  Let $\ndata^*$ be the smallest integer such that
  \begin{equation}
    \frac{\ndata^*}{\beta_{\ndata^*} \gamma_{\ndata^*}} \geq
    \frac{C_1 L_v^2 q |\Rbar_0(\SafeSet_0)|}{\epsilon^2},
  \end{equation}
  then there exists a $\ndata_0 \leq \ndata^*$ such that $\SafeSet_{\ndata_0 + N_{\ndata_0}} = \SafeSet_{\ndata_0}$.
\label{cor:exploration}
\end{corollary}
\begin{proof}
  A direct consequence of~\cref{lem:safeset_complete_at_some_point,cor:Nstar_definition}.
\end{proof}

  %

\subsection{Safety and policy adaptation}

In the following, we denote the true region of attraction of~\cref{eq:dynamic_system} under a policy~$\pi$ by~$\mathcal{R}_\pi$.

\begin{lemma}
  $\Rlevel(\DecSet_\ndata) \subseteq \mathcal{R}_{\pi_n}$ for all $n \geq 0$.
\label{lem:main_rlev_D_contained_in_roa}
\end{lemma}
\begin{proof}
  By definition, we have for all $(\xb, \ub) \in \DecSet_\ndata$ that the exists $(\xb', \ub') \in \SafeSet_{\ndata-1}$ such that
  \begin{align*}
    -L_{\Delta v} \tau &\geq u_\ndata(\xb', \ub') - v(\xb') + L_{\Delta v} \|(\xb, \ub) - (\xb', \ub') \|_1, \\
    &\geq v(f(\xb', \ub')) - v(\xb') + L_{\Delta v} \|(\xb, \ub) - (\xb', \ub') \|_1, \\
    &\geq v(f(\xb, \ub)) - v(\xb),
  \end{align*}
  where the first inequality follows from~\cref{cor:vf_contained_in_C} and the second one by Lipschitz continuity, see~\cref{lem:v_decrease_confidence}.

  By definition of~$\Rlevel$ in~\cref{eq:level_set_operator}, it follows that for all $\xb \in \Rlevel(\DecSet_\ndata) \cap \mathcal{X}_\tau$ we have that $(\xb, \pi_n(\xb)) \in \DecSet_\ndata$. Moreover,~$\Rlevel(\DecSet_\ndata)$ is a level set of the Lyapunov function by definition. Thus the result follows from~\cref{thm:gp_region_of_attraction}.
\end{proof}

\begin{lemma}
    $f(\xb, \ub) \in \mathcal{R}_{\pi_\ndata}\, \forall (\xb, \ub) \in \SafeSet_\ndata$.
\label{lem:main_x_next_contained_in_roa}
\end{lemma}
\begin{proof}
  This holds for~$\SafeSet_0$ by definition. For~$\ndata \geq 1$, by defition, we have for all $\ab \in \SafeSet_n$ there exists an $\ab' \in \SafeSet_{\ndata - 1}$ such that
  \begin{align*}
    \max_{\xb \in \Rlevel(\DecSet_{\ndata})} v(\xb) &\geq u_n(\ab') + L_v L_f \| \ab - \ab' \|_1 \\
    &\geq v(f(\ab')) + L_v L_f \| \ab - \ab' \|_1 \\
    &\geq v(f(\ab))
  \end{align*}
  where the first inequality follows from~\cref{cor:vf_contained_in_C} and the second one by Lipschitz continuity, see~\cref{lem:v_decrease_confidence}. Since $\Rlevel(\DecSet_{\ndata}) \subseteq \mathcal{R}_{\pi_n}$ by~\cref{lem:main_rlev_D_contained_in_roa}, we have that $f(\ab) \in \mathcal{R}_{\pi_n}$.
\end{proof}

\maintheorem*
\begin{proof}
  See~\cref{lem:main_rlev_D_contained_in_roa,lem:main_x_next_contained_in_roa} for~\cref{thm:main_rlev_D_contained_in_roa,thm:main_x_next_contained_in_roa}, respectively. Part~\cref{thm:main_achieve_baseline} is a direct consequence of~\cref{cor:exploration} and~\cref{lem:Sn_contained_in_Rbar0}.
\end{proof}

\end{document}